\documentclass[twoside]{article}

\usepackage[accepted]{aistats2025}
\usepackage[utf8]{inputenc} 

\RequirePackage[backend=biber, style=authoryear, citestyle=authoryear, maxcitenames=2, natbib=true]{biblatex}

\AtEveryBibitem{
    \ifentrytype{book}{}{
        \clearlist{publisher}
        \clearname{editor}
    }
}

\defbibheading{bibliography}{\subsubsection*{References}}

\usepackage[T1]{fontenc}    
\usepackage{hyperref}       
\usepackage{url}            
\usepackage{booktabs}       
\usepackage{amsfonts}       
\usepackage{nicefrac}       
\usepackage{microtype}      
\usepackage{xcolor}         
\usepackage{algorithm}
\usepackage{algorithmic}

\usepackage{graphicx}

\usepackage{amsmath}
\usepackage{amssymb}
\usepackage{mathtools}
\usepackage{amsthm}

\usepackage{centernot}  
\usepackage{xfrac}      
\usepackage{bbm}        
\usepackage{dsfont}        

\usepackage{enumitem}   
\usepackage[capitalize,noabbrev]{cleveref}

\usepackage{tikz}
\usepackage{pgfplots}
\pgfplotsset{compat=1.16}

\usepackage{tikz}
\usepackage{pgfplots}

\usepgfplotslibrary{fillbetween}
\usetikzlibrary{patterns}

\tikzset{
    font={\fontsize{12pt}{12}\selectfont},
}

\pgfplotsset{
    compat=1.5.1,
    primary/.style={color=black, style=solid, line width=1.5pt}, 
    secondary/.style={color=red, style=solid, line width=1.5pt}, 
}

\usetikzlibrary{shapes, arrows}
\usetikzlibrary{decorations.pathreplacing, calligraphy, calc}

\definecolor{BurntOrange}{HTML}{CC5500}
\definecolor{DarkFern}{HTML}{407428}
\definecolor{CBRed}{HTML}{994F00}
\definecolor{CBBlue}{HTML}{006CD1}
\colorlet{Fern}{DarkFern!85!white}
\colorlet{LightFern}{DarkFern!20!white}

\colorlet{LightCerulean}{CBBlue!20!white}

\def\bfI{\mathbf{I}}

\def\bfP{\mathbf{P}}

\def\bbN{\mathbb{N}}

\def\calB{\mathcal{B}}

\def\calL{\mathcal{L}}

\def\calQ{\mathcal{Q}}

\def\calS{\mathcal{S}}
\def\calT{\mathcal{T}}

\def\calX{\mathcal{X}}

\usepackage{thmtools, thm-restate}

\usepackage{mathtools} 

\DeclarePairedDelimiter{\floor}{\lfloor}{\rfloor}

\newcommand{\abs}[1]{\left\vert #1\right\vert}

\newcommand{\rbr}[1]{\left(#1\right)}
\newcommand{\sbr}[1]{\left[#1\right]}
\newcommand{\cbr}[1]{\left\{#1\right\}}
\newcommand{\abr}[1]{\left\langle#1\right\rangle}

\def\norm#1{\|#1\|}

\def\argmax{\mathop{\rm arg\,max}}

\newcommand{\sign}{\text{sign}}

\def\half{\frac 1 2}
\newcommand{\inv}[1]{\frac{1}{#1}}

\newcommand{\into}{\rightarrow}

\newcommand{\R}{\mathbb{R}}

\newcommand{\tL}{\tilde{L}}
\newcommand{\tmu}{\tilde{\mu}}

\newcommand{\w}{w}
\newcommand{\wk}{w_k}
\newcommand{\wkk}{w_{k+1}}
\newcommand{\wopt}{w^*}

\newcommand{\x}{x}
\newcommand{\xk}{x_t}

\newcommand{\xkk}{x_{t+1}}
\newcommand{\xopt}{x^*}
\newcommand{\xbar}{\xbar{w}}

\newcommand{\etak}{\eta_k}

\newcommand{\grad}{\nabla f}

\theoremstyle{plain}
\newtheorem{theorem}{Theorem}[section]
\newtheorem{proposition}[theorem]{Proposition}
\newtheorem{lemma}[theorem]{Lemma}
\newtheorem{corollary}[theorem]{Corollary}
\theoremstyle{definition}
\newtheorem{definition}[theorem]{Definition}

\theoremstyle{remark}
\newtheorem{remark}[theorem]{Remark}

\usepackage[textsize=tiny]{todonotes}

\begin{document}

%

%

\twocolumn[

    \aistatstitle{Level Set Teleportation: An Optimization Perspective}

    \aistatsauthor{Aaron Mishkin \And Alberto Bietti \And Robert M. Gower }

    \aistatsaddress{Stanford University \And  CCM, Flatiron Institute \And CCM, Flatiron Institute} ]


\begin{abstract}
    We study level set teleportation, an optimization routine which tries to
    accelerate gradient descent (GD) by maximizing the gradient norm over a
    level set of the objective.
    While teleportation intuitively speeds-up GD via bigger steps,
    current work lacks convergence theory for convex functions, guarantees for
    solving the teleportation operator, and even clear empirical evidence
    showing this acceleration.
    We resolve these open questions.
    For convex functions satisfying Hessian
    stability, we prove that GD with teleportation obtains a combined
    sub-linear/linear convergence rate which is strictly faster than GD when
    the optimality gap is small.
    This is in sharp contrast to the standard (strongly) convex setting, where
    teleportation neither improves nor worsens convergence.
    To evaluate teleportation in practice, we develop a projected-gradient
    method requiring only Hessian-vector products.
    We use this to show that gradient methods with access to a teleportation
    oracle out-perform their standard versions on a variety of problems.
    We also find that GD with teleportation is faster than truncated Newton
    methods, particularly for non-convex optimization.
\end{abstract}

\section{INTRODUCTION}\label{sec:intro}


We consider minimizing a differentiable function \( f \).
When \( \nabla f \) is \( L \)-Lipschitz continuous (\( L \)-smooth), the
\emph{descent lemma} \citep{bertsekas1997nonlinear} implies
that gradient descent (GD) with step-size \( \etak \) makes progress
proportional to the squared-norm of the gradient,
\begin{equation}
    \label{eq:descent-lemma}
    f(\wkk) \leq f(\wk) - \etak\rbr{1 - \frac{\etak L}{2}}\norm{\grad(\wk)}_2^2.
\end{equation}
If all other quantities are held constant, then increasing the gradient-norm
increases the one-step progress guaranteed by smoothness.
As a result, optimization trajectories which maximize the observed gradients
may converge faster than their naive counterparts.
This same argument has been used to incrementally grow neural
networks~\citep{evci2022gradmax}.

\begin{figure*}[t]
    \centering
    \includegraphics[width=0.48\textwidth]{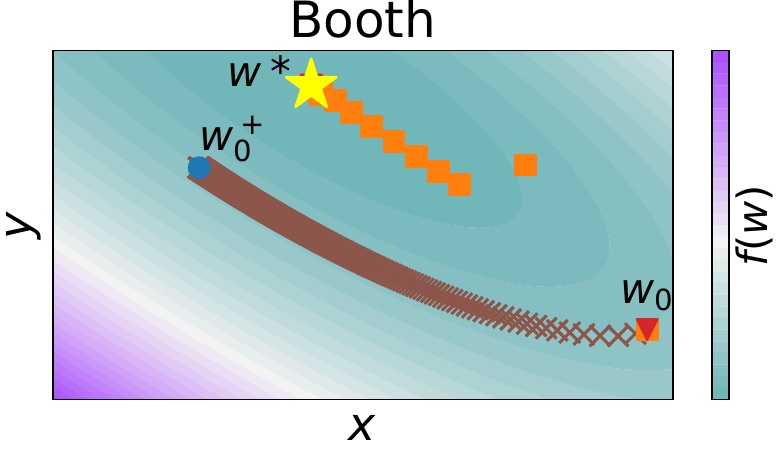}%
    \includegraphics[width=0.48\textwidth]{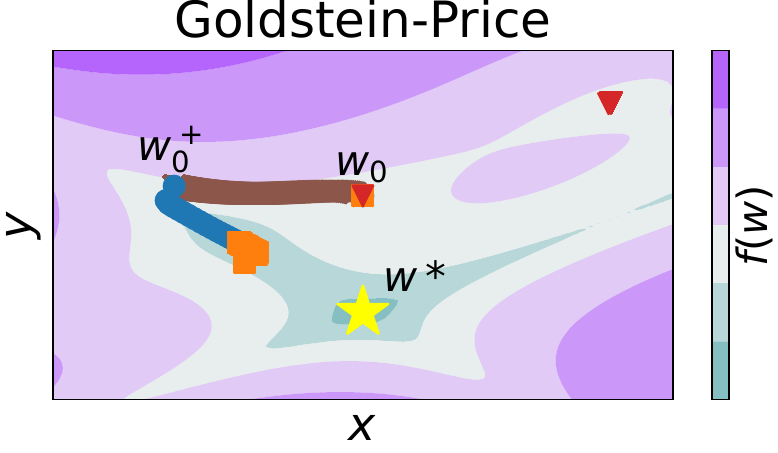}

    \includegraphics[width=0.9\textwidth]{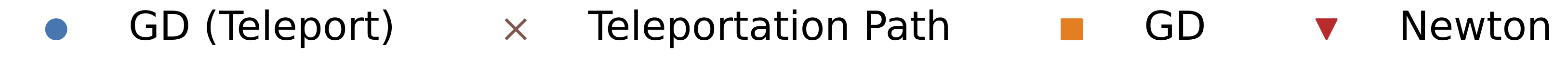}
    \caption{
        Initializing by level set teleportation two test functions.
        The Booth function is a convex quadratic and teleportation aligns $w_0^+$ with
        the maximum eigenvalue-eigenvector pair.
        The next iteration of GD is equivalent to a Newton update and converges in one
        step.
        The Goldstein-Price function is non-convex and teleporting pushes \( w_0^+ \)
        up a narrow ``valley'' from which convergence is slow.
        Newton's method diverges on the non-convex function.
    }%
    \label{fig:toy-teleport}
\end{figure*}

Level set teleportation attempts to leverage the descent lemma by maximizing
the gradient norm without changing the objective value.
At an iteration \( k \) satisfying a pre-determined scheduling rule,
level set teleportation solves the maximization problem,
\begin{equation}
    \label{eq:teleport}
    \wk^+ \in \argmax_w \half \norm{\nabla f(w)}^2_2 \quad \text{s.t.} \quad f(w) = f(w_k),
\end{equation}
where the feasible set \( \calL_k := \cbr{w : f(w) = f(\wk)} \) is the
level set of \( f \) at \( f(\wk) \).
When \( f \) is strongly convex, the Newton and gradient directions at \(
\wk^+ \) coincide \citep{zhao2023symmetry} and the next gradient step is
equivalent to a scaled Newton update (see \cref{fig:toy-teleport}).
For some functions, like quadratics, the gradient norm is also maximized
over the level sets along the gradient flow from \( \wk^+ \) to \( \wopt \)
\citep{zhao2023improving}.

Although these properties suggest that level set teleportation may be an
effective heuristic for improving the convergence rate of gradient methods,
existing theoretical evidence is surprisingly limited.
\citet{zhao2023improving} use the connection to Newton's method to show GD with
teleportation obtains mixed linear/quadratic convergence, but their rate
applies only to strongly convex functions and requires line-search.
In contrast, teleportation for non-convex functions shows no improvement in
convergence speed; rather, it strengthens standard sub-linear rates to
guarantee all gradients in the level set are concentrating
\citep{zhao2023improving}.
To the best of our knowledge, no rates on sub-optimality have been shown for
the convex setting.

Teleportation also suffers from a disconnect between theory and practice:
while existing guarantees require solving \cref{eq:teleport} over the full
level set, empirical studies instead focus on \emph{symmetry teleportation}
\citep{zhao2023symmetry, zhao2023improving}, which restricts optimization to
group symmetries of the objective.
For example, neural networks with positively homogeneous activation functions
are invariant under certain positive rescaling operators.
Although these symmetries may not fully capture \( \calL_k \), optimizing
over parameterized group operators can be very fast and yields an approximate
form of level set teleportation \citep{zhao2023symmetry}.

In this paper, we conduct a rigorous theoretical and empirical evaluation of
level set teleportation for convex functions.
We prove that while teleportation does not impair convergence, additional
assumptions are necessary to obtain improved rates for both convex and
strongly-convex \( f \).
Our novel proof technique, which merges the linear rate from Newton steps
and the sub-linear rate from standard GD, shows that teleportation is most
effective when the optimality gap is small.
This means teleportation is best used when near termination in order to obtain
highly accurate solutions.

In contrast to previous work, we also provide an algorithm for exactly solving
the level set teleportation problem using sequential quadratic programming
(SQP).
Our method only requires Hessian-vector products and resembles a
step of projected GD.
The procedure is parameter-free---the step-size is selected automatically
using a merit function---and convergence is guaranteed via connections to SQP
methods.
Crucially, we leverage our algorithm to provide the first rigorous empirical
evaluation of level set teleportation.
This is summarized along with our other contributions as follows.
\begin{itemize}
    \item We prove that GD with teleportation cannot improve upon the
          convergence rate of standard GD when \( f \) is strongly convex
          unless adaptive step-sizes are used.  In the convex setting, we show
          that the connection to Newton's method may fail and convergence of GD
          with teleportation depends on the diameter of the initial sub-level
          set.

    \item For convex functions satisfying Hessian stability, our novel
          proof combines the sub-linear progress from standard GD steps and
          the linear progress from GD steps after teleporting to obtain a
          convergence rate which is strictly faster than $O(1/K)$.

    \item We develop a tuning-free algorithm for solving teleportation
          problems and use it to show that stochastic and deterministic
          gradient methods converge faster with access to a teleportation
          oracle.
\end{itemize}

\subsection{Additional Related Work}

\textbf{Level Set Teleportation}:
\citet{armenta2020neural} and \citet{armenta2021representation} first used
group symmetries to randomly perturb or ``teleport'' the weights of neural
networks during training.
\citet{zhao2023symmetry} then proposed to optimize over parametric
symmetries for which teleportation is fast or even closed form, while
\citet{zhao2023flat} studied more sophisticated data-dependent symmetries.
While our focus is on convex problems, \citet{zhao2023improving} analyze
convergence of GD with teleportation for general non-convex functions and show
that the gradient norm converges uniformly over the sub-level set,
strengthening standard guarantees.
They also give rates under the PL condition \citep{karimi2016pl}.

\textbf{Symmetries in Optimization}: Teleportation is closely connected to
sharp/flat minima in deep learning \citep{hochreiter1997flat,
    keskar2017sharp, dinh2017sharp}.
Sharpness aware minimization \citep{foret2021sharpness} biases optimization
towards regions with low curvature, while level set teleportation actively
finds large gradients.
In contrast to GD with teleportation, \citet{neyshabur2015pathsgd} suggest
model invariances harm optimization and propose Path-SGD, a gradient method
which is invariant to rescaling symmetries in neural networks.
Similarly, \citet{bamler2018improving} leverage group operators by separating
optimization into directions with symmetries and those without.

\textbf{Initialization}:
While initialization methods for neural networks typically balance layer norms
to prevent exploding or vanishing gradients \citep{narkhede2022review},
teleportation skews layers to maximize the gradient norm.
Our work is related to \citet{saxe2014initialization, min2021initialization},
who study initialization in linear networks, and
\citet{tarmoun2021understanding}, who propose symmetry operators for
initializing linear models with two layers.
When used for initialization, teleportation can be thought of as data
dependent pre-training \citep{mikko1998weight, bengio2006greedy}.

\section{LEVEL SET TELEPORTATION}\label{sec:teleport}


\begin{algorithm}[t]
	\caption{GD with Teleportation}%
	\label{alg:gd-with-teleport}
	\begin{algorithmic}
		\STATE \textbf{Inputs}: \( w_0 \); step-sizes \( \etak \); block indices \( \calB \), sizes \( b_k \).
		\STATE \( \calT \gets \bigcup_{k \in \calB} \cbr{k, k+1, \ldots, k+b_k-1} \)
		\FOR{\( k \in \cbr{0, \ldots, K} \)}
		\IF{\( k \in \calT \)}
		\STATE \( w_k^+ \in \argmax \cbr{\norm{\grad(w)}_2 : f(w) \leq f(\wk)} \)
		\ELSE
		\STATE \( w_k^+ \gets w_k \)
		\ENDIF
		\IF{line-search}
		\STATE \( g_k \gets \norm{\grad(\wk^+)}_2^2 \)
		\STATE \( \etak \!\gets\! \max \cbr{\eta\!:\! f(\wk^+ \!-\! \eta \grad(\wk^+)) \!\leq\! f(\wk^+) \!-\! \frac{\eta}{2} g_k } \)
		\ENDIF
		\STATE \( \wkk \gets \wk^+ - \etak \grad(\wk^+) \)
		\ENDFOR
		\STATE \textbf{Output}: \( w_K \)
	\end{algorithmic}
\end{algorithm}

We begin by considering the convergence of gradient methods with intermittent
teleportation.
We assume that \( f \) is \( L \)-smooth, has at least one minimizer \( \wopt
\), and is coercive.
Coercivity holds if and only if the level sets of \( f \) are bounded and is
sufficient to guarantee that the teleportation problem admits a finite
solution;
this property is easy to check in practice since boundedness of one level set
is sufficient for all level sets to be bounded when \( f \) is convex
\citep[Prop. 1.4.5]{bertsekas2009convex}.
We show in \cref{prop:relu-teleport} that teleportation is ill-posed for
non-coercive neural network problems.

Instead of solving \cref{eq:teleport}, we focus on
the more general sub-level set teleportation problem,
\begin{equation}
    \label{eq:sublevel-teleport}
    w_{k}^+ = \argmax_{w} \half \norm{\grad(w)}_2^2 \quad \text{s.t.} \quad f(w) \leq f(\wk),
\end{equation}
where the feasible set is \( \calS_k = \cbr{w : f(w) \leq f(\wk)} \).
For convex \( f \), \cref{eq:sublevel-teleport} admits at least one solution
on the boundary of \( \calS_k \); if \( f \) is strictly convex, then every
solution is on the boundary and the relaxation is equivalent to level set
teleportation (\cref{lemma:relaxation-solutions}).
However, sub-level set teleportation is acceptable even for general functions
since our goal is to minimize \( f \).

Although sub-level set teleportation can be used with any iterative optimizer,
we focus on (stochastic) gradient methods.
To simplify the presentation, we group teleportation steps into non-overlapping blocks.
Let \( \calB \) be the starting iterations of the blocks, \( b_k \) the length
of the block starting at iteration \( k \), and \(\calT \) the complete
teleportation schedule (tele-schedule) so that
\( k \in \calB \) implies \( k, k+1, \ldots, k+b_k-1 \in \calT \).
Non-overlapping means \( k \in \calB \) implies \( k - 1 \not \in \calT \).
See \cref{alg:gd-with-teleport}.

Given initialization \( w_0 \), let \( \wopt \) be the
projection of \( w_0 \) onto the optimal set.
The initial sub-level set \( \calS_0 = \cbr{w : f(w) \leq f(w_0)} \) and
diameter \( R = \sup \cbr{\norm{\wopt - w}_2 : w \in \calS_0} \) play a
critical role in our analysis.
This is because sub-level set teleportation is not non-expansive
(\cref{def:non-expansive}), meaning
\( \norm{w_k^+ - \wopt}_2 > \norm{w_k - \wopt}_2 \) may occur.
However, the optimality gaps \( \delta_k := f(\wk) - f(\wopt) \)
contract if the step-sizes are sufficiently small, so the
iterates remain in \( \calS_0 \).

\begin{restatable}{lemma}{diameterBound}\label{lemma:diameter-bound}
    If \( f \) is \( L \)-smooth and \( \etak < 2/L \), then GD
    with tele-schedule \( \calT \) satisfies
    \( \delta_{k+1} \leq \delta_k \), \( \wk \in \calS_0 \), and
    \( \norm{\wk - \wopt}_2 \leq R \) for every \( k \in \bbN \).
\end{restatable}
If \( f \) is \( \mu \)-strongly convex (\( \mu \)-SC), meaning
\( \forall \, w, w' \in \R^d \),
\begin{equation}
    \label{eq:strong-convexity}
    f(w) \geq f(w') + \abr{\grad(w'), w - w'} + \frac{\mu}{2} \norm{w - w'}_2^2,
\end{equation}
then we also have progress in \( \Delta_k = \norm{\wk - \wopt}_2^2 \).
\begin{restatable}{lemma}{scProgress}\label{lemma:sc-progress}
    Suppose \( f \) is \( L \)-smooth and \( \mu \)-SC.
    If \( k \in \calT \),
    then GD with step-size \( \etak \) satisfies
    \begin{equation}
        \Delta_{k+1}
        \leq
        \max\cbr{(1 \!-\! \eta_k L)^2, (1 \!-\! \eta_k \mu)^2}\\
        \norm{\wk^+ - \wopt}_2^2.
    \end{equation}
\end{restatable}

Unfortunately, \( \norm{\wk^+ - \wopt}_2^2 \) and \( \Delta_k \) differ by a
factor of \( L / \mu \) in the worst case since teleportation is expansive,
leaving us to prove a weaker rate in the general setting.

\begin{restatable}{proposition}{scTeleport}\label{prop:sc-teleport}
    If \( f \) is \( L \)-smooth and \( \mu \)-SC,
    then GD with \( \etak < 2 / L \) and
    tele-schedule \( \calT \) satisfies,
    \begin{equation}
        \label{eq:slow-linear-rate}
        \delta_K \leq \prod_{i=0}^{K-1}
        \sbr{1 - 2 \mu \eta_i \rbr{1 - \frac{\eta_i L}{2}}}
        \delta_0.
    \end{equation}
    Moreover, if the teleportation operator is non-expansive, then for any
    step-sizes \( \etak \) we have,
    \begin{equation}
        \label{eq:quadratic-rate}
        \delta_K \leq \frac{L}{\mu} \prod_{i=0}^{K-1} \sbr{
            \max\cbr{(1 - \eta_i L)^2, (1 - \eta_i \mu)^2}
        }
        \delta_0.
    \end{equation}
    Finally, there exists \( f \) for which~\eqref{eq:quadratic-rate} holds and
    is tight.
\end{restatable}

\cref{eq:quadratic-rate} is the tight minimax rate for GD on a strongly-convex
function \citep[Section 1.3]{bertsekas1997nonlinear};
since it is also tight for GD with teleportation, \cref{prop:sc-teleport} shows
that a teleportation oracle cannot accelerate GD in the worst case.
This is consistent with \citet{zhao2023improving}, who
require adaptive step-sizes from a Wolfe line-search
\citep{wolfe1969convergence, wolfe1971convergence} to show a linear/quadratic
rate.
Before giving conditions under which teleportation does strictly improve
convergence, we show that the connection to Newton's method may fail for
non-strongly convex \( f \).

\subsection{Convergence for Convex Functions}\label{sec:convex-functions}

The link between GD with teleportation and Newton's method follows from
optimality conditions for teleportation.
Let \( \wk^+ \) be a local maximum of \cref{eq:sublevel-teleport}.
If the linear independence constraint qualification (LICQ) holds, then there
exists \( \lambda_k \geq 0 \) such that \( \wk^+ \) satisfies the KKT
conditions~\citep{bertsekas1997nonlinear},
\begin{equation}
    \label{eq:kkt-conditions}
    \begin{aligned}
        \nabla^2 f(\wk^+) \grad(\wk^+) - \lambda_k \grad(\wk^+) = 0, \\
        f(\wk^+) \leq f(\wk), \, \,
        \text{and} \, \, \lambda_k (f(\wk^+) - f(\wk)) = 0.
    \end{aligned}
\end{equation}
\cref{eq:kkt-conditions} reveals that the dual parameter \( \lambda_k \)
is also an eigenvalue of \( \nabla^2 f(\wk^+) \).
If \( \lambda_k \neq 0 \), then
\begin{equation}
    \label{eq:newton-direction}
    \grad(\wk^+) = \lambda_k \rbr{\nabla^2 f(\wk^+)}^{\dagger} \grad(\wk^+),
\end{equation}
where \( A^\dagger \) denotes the pseudo-inverse of \( A \).
Thus, the first step of GD after teleporting is a Newton step with unknown
scale \( \lambda_k \).
Since the teleportation problem satisfies LICQ unless \( \grad(\wk^+) = 0 \)
(in which case \( f \) is stationary over \( \calS_k \)), the connection to
Newton may only fail when \( \nabla^2 f \) is positive semi-definite and \(
\lambda_k = 0 \).
This is possible when \( f \) is non-strongly convex.

\begin{restatable}{proposition}{newtonDirection}\label{prop:newton-direction}
    There exists an \( L \)-smooth, convex, and \( C^2 \) function for which the
    gradient direction after teleporting is not the Newton direction.
\end{restatable}

Without the connection to Newton's method, we can only prove that GD with
teleportation attains the same order of convergence as GD without
teleportation.

\begin{restatable}{proposition}{convexTeleport}\label{prop:convex-teleport}
    If \( f \) is \( L \)-smooth and convex, then GD with \( \eta < 2 /
    L \) and tele-schedule \( \calT \) satisfies
    \[
        \delta_K \leq 2 R^2 / \rbr{K \eta \rbr{2 - L \eta}}.
    \]
    Moreover, there exists \( f \) for which the convergence of GD with
    and without teleportation is identical.
\end{restatable}

In contrast, \citet[Theorem 3.3]{bubeck2015convex} prove that standard GD
converges at the rate,
\[
    \delta_k \leq 2 \norm{\w_0 - \wopt}_2^2 \, / \rbr{K \eta \rbr{2 - L \eta}}.
\]
Compared with standard results, GD with teleportation depends
on the diameter of the initial sub-level set, \( R^2 \), rather than the
initial distance \( \norm{\w_0 - \wopt}_2^2 \).
This dependence is necessary since one step of teleportation can nearly
attain the diameter of the sub-level set.
\begin{restatable}{proposition}{teleportDistances}\label{prop:teleport-distances}
    There exists an \( L \)-smooth and convex function such that teleporting
    at \( k = 0 \) guarantees,
    \[
        \norm{\w_0^+ - \wopt}_2 \geq R / 4.
    \]
\end{restatable}

\subsection{Convergence under Hessian Stability}\label{sec:stability}

So far we have shown that teleportation does not improve the convergence order
of GD in the general convex setting \emph{and} may blow-up the initial
distance to \( \wopt \).
However, in the special case where the Hessian is well-behaved, we show that
teleportation leads to strictly faster optimization.
We quantify the behavior of the Hessian using the notion of stability
\citep{bach2010self, karimireddy2019newton, gower2019newton}.

\begin{definition}\label{def:hessian-stability}
    We say \( f \) is \( (\tL, \tmu) \) Hessian stable
    over \( \calQ \) if \( \forall \, x, y \in \calQ \),\,
    \( \nabla^2 f(x) (y - x) \neq 0 \) and,
    \begin{align}
        f(y)
         & \leq f(x) + \abr{\grad(x), y \!-\! x}
        + \frac{\tL}{2} \norm{y \!-\! x}^2_{\nabla^2 f(x)},%
        \label{eq:stability-upper-bound}         \\
        f(y)
         & \geq f(x) + \abr{\grad(x), y \!-\! x}
        + \frac{\tmu}{2} \norm{y \!-\! x}^2_{\nabla^2 f(x)}.%
        \label{eq:stability-lower-bound}
    \end{align}
\end{definition}

Hessian stability holds for many problems, including logistic
regression \citep{bach2010self}.
Choosing \( \calQ = \calS_0\) and using \cref{lemma:diameter-bound}, shows
that Hessian stability holds at every pair \( (\wkk, \wk^+) \).
As a result, \( \nabla^2 f(\wk^+) \grad(\wk^+) \neq 0 \) and
\cref{eq:newton-direction} implies the Newton and gradient directions are
collinear.
In particular, this means the first step of GD after teleporting obtains a
linear rate.
\begin{restatable}{lemma}{stabilityTeleportProgress}\label{lemma:stability-teleport-progress}
    Suppose \( f \) is \( L \)-smooth, convex, and satisfies Hessian stability on
    \( \calS_0 \).
    Then the first step of GD after teleporting with step-size
    \( \eta < 2 / L \tL \) makes linear progress as
    \begin{equation}
        \label{eq:stability-teleport-progress}
        \delta_{k+1}
        \leq \rbr{1 + \tmu \lambda_k \eta
            \rbr{\eta \lambda_k \tL - 2}} \delta_k.
    \end{equation}
\end{restatable}
Choosing the ideal step-size \( \eta = 1/(\lambda_k \tL) \) gives \( \delta_{k+1}
\leq (1- \tmu / \tL) \delta_k \), but in general we cannot know \( \lambda_k
\).
One way around this difficulty is to use the Armijo line-search
\citep{armijo1966minimization} after teleporting,
\begin{equation}
    \label{eq:armijo-ls}
    f(\wkk) \leq f(\wkk^+) - \frac{\eta}{2} \norm{\grad(\wk^+)}_2^2.
\end{equation}
Choosing \( \eta \) to be the largest step-size satisfying \cref{eq:armijo-ls}
gives the following progress condition:
\begin{restatable}{lemma}{stabilityTeleportProgressLS}\label{lemma:stability-teleport-progress-ls}
    Suppose \( f \) is \( L \)-smooth, convex, and satisfies Hessian stability on
    \( \calS_0 \).
    Then the first step of GD with Armijo line-search after teleporting
    satisfies,
    \begin{equation}
        \label{eq:stability-teleport-progress-ls}
        \delta_{k+1} \leq
        \rbr{1 - \tmu / \tL}
        \delta_k.
    \end{equation}
\end{restatable}

\begin{figure*}[t]
    \centering
    \begin{tikzpicture}[scale=1,
    ]
    \begin{axis}[width=0.95\linewidth,
        height=0.33\linewidth,
        axis lines=none,  
        yticklabels={,,}, xticklabels={,,},
        ymin=-2, ymax=2.1,
        xmin=-4, xmax=4,
        view={0}-{90},
        ]

        \addplot3[
            domain=-4:4,
            domain y=-4:4,
            contour gnuplot={
                    labels=false,
                    levels={1, 4, 10},
                    draw color=black,
                },
            line width=2pt,
            samples=100,
        ] {x^2 + 4*y^2 };

        \addplot[
            domain=-4:4,
            samples=10,
            line width=2pt,
            draw=red,
        ] {-x/2 + 2/sqrt(2)};

        \addplot[
            ->,
            domain=1.414:1.545,
            samples=10,
            line width=2pt,
            draw=blue,
        ] {8*x - 15/sqrt(2)};

        \addplot[
            domain=-4:4,
            samples=10,
            line width=2pt,
            draw=purple,
        ] {1.0227 + (4 - 2*0.783*(x-0.783) - 0.783^2 - 4*1.0227^2)/(8*1.0227)};

        \node[
        star,
        fill=black,
        inner sep=0pt,
        minimum size=6pt,
        label={[label distance=-0.4mm]3:{$\wopt$}}
        ] (opt) at (axis cs:0,0) {};

        \node[
        label={[label distance=1.9mm]90:{$F(\w)$}}
        ] at (axis cs:1,-0.5) {};

        \node[
        circle,
        fill=black,
        inner sep=0pt,
        minimum size=6pt,
        label={[label distance=-1mm]90:{$x_0$}}
        ] (a) at (axis cs:1.414,0.707) {};

        \node[
        circle,
        fill=black,
        inner sep=0pt,
        minimum size=6pt,
        label={[label distance=1.3mm]180:{$x_0^+$}}
        ] (b) at (axis cs:1.55,1.79) {};

        \node[
        circle,
        fill=black,
        inner sep=0pt,
        minimum size=6pt,
        label={[label distance=0mm]90:{$x_1$}}
        ] (c) at (axis cs:0.783,1.0227) {};

        \draw[
            ->,
            draw=orange,
            line width=2pt,
        ] (b) -- (c);

        \node[
        label={[label distance=-1mm, text=red]90:{$l_0$}}
        ] at (axis cs:2.5,0.164) {};

        \node[
        label={[label distance=-0.8mm, text=purple]90:{$l_1$}}
        ] at (axis cs:-1.35,1.4) {};

        \draw[
        draw=none,
        ] (a) -- (b)
        node[pos=0.9, right, text=blue,
        label={[label distance=-0.2mm, text=blue, ]0:{$\nabla^2 F(x_0) \nabla F(x_0)$}}
        ]{};

        \draw[
        draw=none,
        ] (b) -- (c)
        node[pos=0.80, right,
        label={[label distance=0.2mm, text=orange, ]0:{$v_{0}$}}
        ]{};



    \end{axis}
\end{tikzpicture}%
    \caption{
        One iteration of our method for solving teleportation problems
        on a convex quadratic.
        The algorithm combines gradient ascent with projections onto
        the linearization
        \( l_t = \cbr{w : f(\xk) + \abr{\grad(\xk), w - \xk} = f(\wk)} \).
    }%
    \vspace{-2ex}%
    \label{fig:update}
\end{figure*}

Converting Lemmas~\ref{lemma:stability-teleport-progress}
and~\ref{lemma:stability-teleport-progress-ls} into a fast convergence rate
requires merging the linear progress from steps after teleporting with
the sub-linear progress from standard GD.
Our proof technique unrolls the linear rate across adjacent iterations
with teleportation and then bounds \( \delta_k \) using the descent lemma to
yield a combined convergence rate.
In what follows, we assume the step-sizes are selected by line-search; see
\cref{thm:stability-teleport} for an alternative result using a fixed
step-size.
\begin{restatable}{theorem}{stabilityTeleportLS}\label{thm:stability-teleport-ls}
    Suppose \( f \) is \( L \)-smooth, convex, and satisfies Hessian stability on
    \( \calS_0 \).
    Consider any tele-schedule \( \calT \) and let \( M = K - \abs{\calT} \) be
    the number of steps without teleportation.
    Then GD with step-size chosen by Armijo line-search converges as,
    \begin{equation}
        \label{eq:stability-teleport-ls}
        \delta_K \leq
        \frac{2 R^2 L}
        {M + 2 R^2 L \sum_{k \in \calB}
            \sbr{\rbr{\frac{\tL}{\tL - \tmu}}^{b_k} - 1} \inv{\delta_{k-1}}}.
    \end{equation}
\end{restatable}
Assuming we teleport every-other iteration (\( b_k = 1 \)) allows us to
specialize \cref{thm:stability-teleport-ls} to obtain
\begin{equation}
    \delta_K \leq
    \frac{2 R^2 L}
    {K/2 +  \frac{2 R^2 L \tmu}{\tL - \tmu} \sum_{k \in \calT}
        \inv{\delta_{k-1}}},
\end{equation}
which shows teleporting at iteration \( k \) leads to a strictly faster rate
than standard GD if \( \delta_{k-1} \leq (2 R^2 L \tmu)/(\tL - \tmu) \).
That is, teleportation is effective when the optimality gap is small.
This reflects the typical ordering between linear and sub-linear convergence,
where the sub-linear rate is faster than the linear rate for a finite number of
initial iterations \citep{bach2024scaling}.
At the cost of hiding that relationship, weighted telescoping gives a fully
explicit rate for GD with teleportation.

\begin{restatable}{theorem}{stabilityCombinedLS}\label{thm:stability-combined-ls}
    Suppose \( f \) is \( L \)-smooth, convex, and satisfies Hessian stability on
    \( \calS_0 \).
    Then GD with line-search and any tele-schedule \( \calT \)
    converges as,
    \begin{equation}
        \label{eq:combined-rate-ls}
        \delta_K
        \leq 2 R^2 L \bigg[\sum_{k \not \in \calT}
            \rbr{\frac{\tL}{\tL - \tmu}}^{n_k}\bigg]^{-1},
    \end{equation}
    where \( n_k =  |\cbr{ i \in \calT : i > k}| \) is the number
    teleportation steps after iteration \( k \).
\end{restatable}

Compared to \cref{thm:stability-teleport-ls}, this rate does not depend on
the optimality gaps for progress and is an explicit convergence guarantee.
As a price, the progress no longer decomposes into a sum over teleportation
steps and a sum over regular GD steps.
Despite this downside, it is still useful to interpret the rate
in the special case of teleporting every-other iteration, for which we obtain,
\begin{equation}\label{eq:alternating-linear-rate}
    \delta_K
    \leq
    2 R^2 L \frac{(\tL - \tmu)}{\tmu}\sbr{\rbr{\frac{\tL}{\tL - \tmu}}^{K/2} - 1}^{-1}.
\end{equation}
This surprising result shows that teleportation allows GD to obtain a linear
rate without strong convexity and implies teleportation is
particularly useful for computing high-accuracy solutions.
However, we expect \cref{eq:alternating-linear-rate} and
\cref{thm:stability-combined-ls} to be quite loose in general since they
strongly down-weight steps without teleportation, effectively ignoring the
progress made by standard GD steps.
See \cref{cor:alternating-steps-ls} for details and
\cref{thm:stability-teleport-combined} for fixed step-sizes.

\begin{remark}
    Theorems~\ref{thm:stability-teleport-ls}
    and~\ref{thm:stability-combined-ls} apply (with minor modifications to the
    proofs) to schemes which combine GD steps with intermittent standard Newton
    updates.
    As far as we know, these are the first analyses for alternating
    first-order/second-order methods which combine sub-linear and linear
    progress conditions to obtain a non-trivial convergence rate.
    In particular, using Newton steps at every iteration allows us to exactly
    match the rates for Newton's method \citep{karimireddy2019newton} or
    randomized subspace Newton \citep{gower2019newton} depending on the update
    scheme used.
\end{remark}



\section{EVALUATING THE TELEPORTATION OPERATOR}\label{sec:algorithm}


\begin{figure*}[t]
    \centering
    \begin{minipage}{0.50\textwidth}
        \input{figures/teleport_algo}
    \end{minipage}%
    \begin{minipage}{0.50\textwidth}
        \begin{figure}[H]
            \centering
            \includegraphics[width=\textwidth]{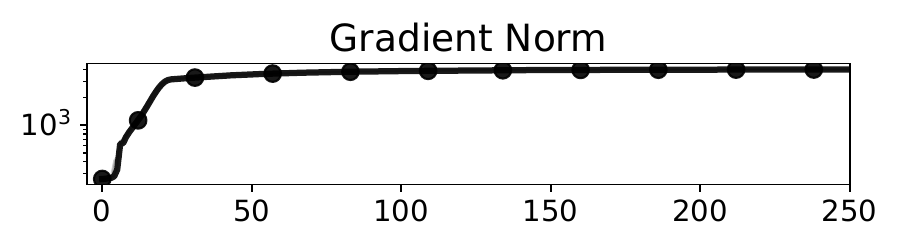}
            \includegraphics[width=\textwidth]{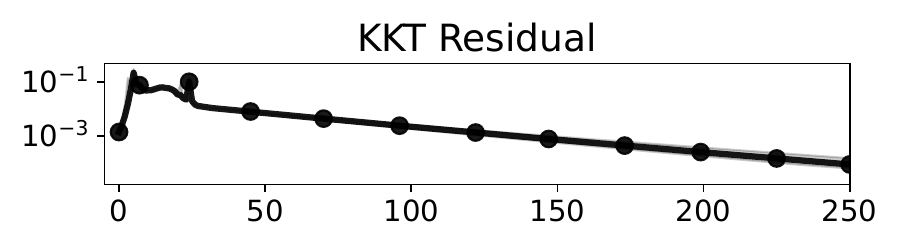}
            \includegraphics[width=\textwidth]{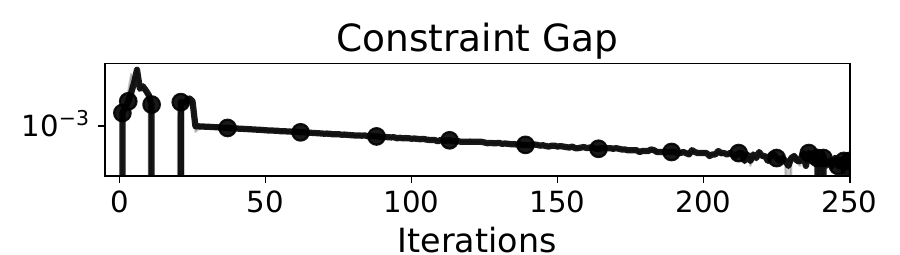}
        \end{figure}
    \end{minipage}
    \caption{
        \textbf{Left}: full algorithm for sub-level set teleportation.
        \textbf{Right}: sub-level set teleportation for a two-layer MLP with 50 hidden
        units on MNIST.
        Our algorithm gives an approximate KKT point where the gradient norm is
        two orders of magnitude larger than that at the standard
        initialization.
    }%
    \label{fig:tele-metrics}
    \vspace{-2ex}
\end{figure*}

The results in the previous section assume access to a
\emph{teleportation oracle}, meaning a function which can solve the general
non-linear programming problem needed for sub-level set teleportation.
Although we could apply standard SQP methods, computing the Hessian
of the teleportation objective requires third-order derivatives of \( f \),
which is not feasible for large-scale problems.
Instead, we develop an iterative, projected-gradient-type method that requires
only Hessian-vector products.
We denote by \( x_t \) the iterates of our method for solving teleportation
problems, with \( x_0 = w_k \).

For general \( f \), the sub-level set \( \calS_k \) is non-convex with no
closed-form projection.
In contrast, the linearization of this constraint around
\( \xk \) yields a half-space,
\[
    \tilde \calS_k(\xk) :=
    \cbr{w : f(\xk) + \abr{\grad(\xk), w - \xk} \leq f(\wk)},
\]
for which projections are simple.
To obtain a tractable algorithm, we maximize a penalized linearization of the
log-gradient-norm subject to this constraint,
\begin{equation}\label{eq:teleportation-sub-problem}
    \begin{aligned}
        \xkk
         & \!=\! \argmax_{x \in \tilde \calS_k(\xk)} \bigg\{\half \log(\norm{\grad(\xk)}_2^2)
        \!-\! \frac{1}{2 \rho_t} \norm{\x \!-\! \xk}_2^2                                             \\
         & \hspace{3em} + \abr{\frac{\nabla^2 f(\xk) \nabla f(\xk)}{\norm{\grad(\xk)}_2^2}, x - \xk}
        \bigg\}.
    \end{aligned}
\end{equation}
Taking the logarithm of the objective encodes positivity of the
gradient norm and leads to a normalized update rule.
Define \( q_t = \nabla^2 f(\xk) \grad(\xk) \) and
\( g_t = \norm{\grad(\xk)}_2^2 \).
\begin{restatable}{proposition}{updateRule}\label{prop:update-rule}
    \cref{eq:teleportation-sub-problem} is solved by,
    \begin{equation}\label{eq:update-rule}
        \begin{aligned}
             & \!v_t
            = \rbr{\rho_t \abr{\grad(\xk), q_t} \!/\! g_t
            \!+\! f(\xk) \!-\! f(\wk)}_+ \!\grad(\xk), \\
             & \!\xkk
            = \xk + \rbr{\rho_t q_t - v_t}/g_t.
        \end{aligned}
    \end{equation}
\end{restatable}
This iteration is equivalent to projected GD with a linearized sub-level set
constraint (see \cref{fig:update}).
It is also equivalent to an SQP method which uses \( \bfI/\rho_t \) as an
approximation to the Hessian of the teleportation objective \citep{torrisi2018projected}.
We leverage SQP to obtain the following guarantee.
\begin{restatable}{theorem}{sqpConvergence}\label{thm:sqp-convergence}
    Assume \( \grad(\wk) \neq 0 \), strict complementarity, and that
    \cref{eq:update-rule} is used with the relaxation step
    \( \hat{\x}_{k+1} \!=\! \alpha_t \xkk \!+ (1 \!-\! \alpha_t) \xk \).
    Then, using appropriate \( (\alpha_t, \rho_t) \),
    \( \hat{\x}_{k} \) converges asymptotically to a KKT
    point \( (x^*\!, \lambda^*) \) of \cref{eq:sublevel-teleport};
    if \( (x^*\!, \lambda^*) \) is second-order critical,
    then local convergence is linear.
\end{restatable}%
Global linear convergence also holds under additional assumptions
\citep{torrisi2018projected} --- see \cref{fig:tele-metrics} for a real-world
example.
Although formally necessary, we have not found relaxation steps to be useful in
practice.
Finally, we remark that the requirement for strict complementarity is standard
in non-linear programming \citep{de2019strict} and known to be satisfied for
generic conic programs \citep{pataki2001generic}.

\begin{figure*}[t]
    \centering
    \includegraphics[width=0.98\textwidth]{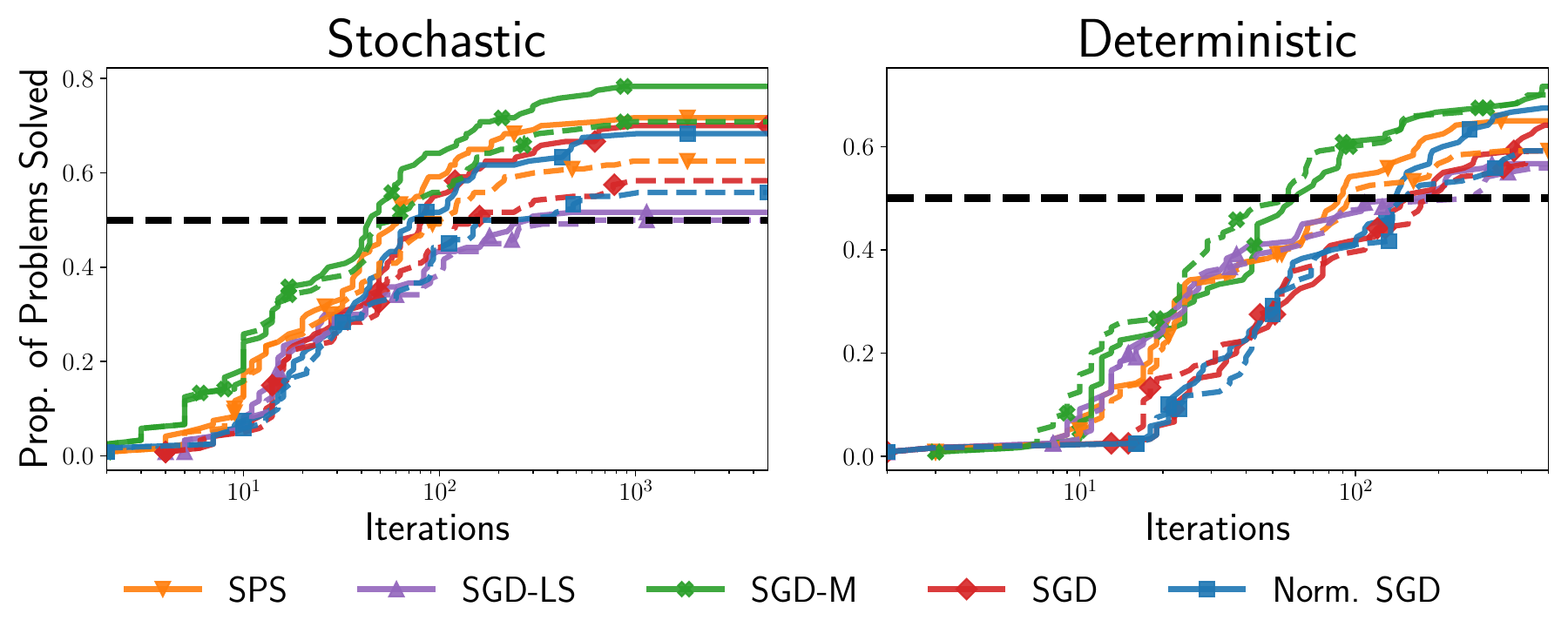}
    \caption{
        Performance profile comparing optimization methods with (solid lines)
        and without (dashed lines) sub-level set teleportation for training
        three-layer ReLU networks.
        Stochastic methods teleport once every \( 10 \) epochs starting
        from the epoch \( 5 \), while
        deterministic methods teleport once every \( 50 \) iterations
        starting from \( k=5 \).
        A problem is ``solved''' when \( \rbr{f(\wk) - f(\wopt)} /
        (f(w_0) - f(\wopt)) \leq \tau \), where \( f(\wopt) \) is estimated
        separately and \( \tau \) is a threshold.
        Performance is judged by comparing time to a
        fixed proportion of problems solved (see dashed line at 50\%).
        Algorithms with intermittent teleportation uniformly dominate their
        standard counterparts.
    }%
    \label{fig:network-profile}
    \vspace{-2ex}
\end{figure*}

\subsection{Step-sizes and Termination}

A disadvantage of our update is that it requires a step-size \( \rho_t > 0 \).
Since teleportation is a sub-routine of a larger optimization procedure,
tuning \( \rho_t \) for good performance is not acceptable.
Following practice for SQP methods
(see, e.g., \citet[Theorem 18.2]{nocedal1999numerical}),
we select \( \rho_t \) using line-search on a \emph{merit function}
\( \phi_\gamma(\x) = \half \log(\norm{\nabla f(\x)}_2^2)
- \gamma \rbr{f(x) - f(\wk)}_+ \),
\begin{equation}\label{eq:line-search}
    \phi_\gamma(\xkk) \geq \phi_\gamma(\xk) + D_\phi(\xk, \xkk - \xk) / 2,
\end{equation}
where \( \gamma > 0 \) is the penalty strength, and \( D_\phi(\xk, u) \) is the
directional derivative of \( \phi_\gamma \) evaluated at \( \xk \) in direction
\( u \).
Setting \( \gamma \) sufficiently large guarantees that
Eq.~\eqref{eq:update-rule} gives an ascent direction of \( \phi_\gamma \).
\begin{restatable}{proposition}{penaltyStrength}\label{prop:penalty-strength}
    If \( f(\xk) > f(\wk) \) and
    \( \gamma_t > \abr{q_t, v_t} / \norm{\grad(\xk)}_2^4 (f(\xk) - f(\wk))_+ \),
    then \( \xkk - \xk \) is an ascent direction of \( \phi_{\gamma_t} \) and
    Eq.~\eqref{eq:line-search} simplifies to,
    \begin{equation}
        \phi_{\gamma_t}(\xkk) \geq
        \phi_\gamma(\xk) + \rho_t \norm{q_t}_2^2/2 g_t^2.
    \end{equation}
\end{restatable}
\cref{prop:penalty-strength} provides a recipe for computing
step-sizes using backtracking line-search.
Note that when \( v_t = 0 \) the update reduces to gradient ascent and
\( \gamma_t = 0 \) is immediately sufficient for progress.

We also briefly discuss how to terminate our algorithm.
Let \( \bfP_t \) be the projection onto the orthogonal complement of
\( \grad(\xk) \).
The KKT conditions (\cref{eq:kkt-conditions}) imply \( \xk \) is
near-optimal when \( \norm{\bfP_t q_t} < \epsilon \) and
\( f(\xk) - f(\wk) \leq \delta \).
\cref{alg:teleport} combines these conditions with line-search to give a
complete solver.

\subsection{Approximate Teleportation}\label{sec:approximate-teleportation}

\cref{alg:teleport} is a practical and theoretically justified
procedure for teleporting, but --- like other SQP methods --- it requires
unbounded iterations to compute an exact solution to \cref{eq:sublevel-teleport}.
Instead, we typically obtain approximate KKT points satisfying,
\begin{equation}\label{eq:approximate-kkt}
    \begin{aligned}
        \nabla^2 f(\wk^+) \grad(\wk^+) - \lambda_k \grad(\wk^+) = e_k, \\
        f(\wk^+) - f(\wk) \leq r_k,
    \end{aligned}
\end{equation}
where \( \norm{e_k}_2 \leq \epsilon_k \) is an error vector.
Thus, as a corollary of \cref{lemma:stability-teleport-progress},
GD with approximate teleportation and step-size \( \eta \) makes at
least as much progress as,
\[
    \delta_{k+1}
    \leq \rbr{1 \!+\! \tmu \lambda_k \eta(\eta \lambda_k \tL \!-\! 2)} \delta_k
    + \frac{\tL \eta^2}{2} \norm{e_k^+}_{\nabla^2 f_k}^2 + r_k.
\]
While the \( e_k \) term can be managed using a decreasing step-size
schedule, \( r_k \) can only be controlled by increasingly strict tolerances
for the teleportation algorithm.
If \( r_k \rightarrow 0 \) at an appropriate rate, then it is straightforward
to show convergence of the overall GD scheme \citep{bertsekas2000gradient}.
Yet, this still does not lead to a combined gradient/Hessian complexity for GD
with approximate teleportation because the convergence guarantee for
\cref{alg:teleport} is asymptotic.
We leave complete iteration complexities to future work.


\section{EXPERIMENTS}\label{sec:experiments}


\begin{figure*}[t]
    \centering
    \includegraphics[width=0.98\textwidth]{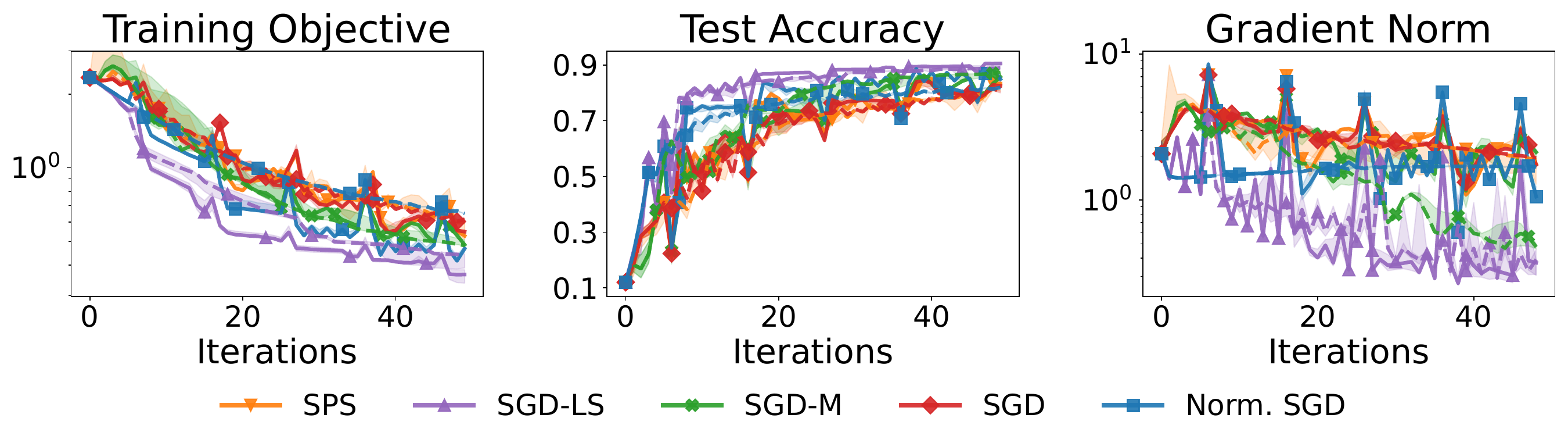}
    \caption{Performance of optimizers with (solid) and without (dashed)
        teleportation on MNIST.  We train a MLP with the soft-plus activation
        and one hidden layer of size \( 500 \).
        All methods are run in batch mode.
        Teleportation significantly improves the convergence speed of
        all methods and does not affect generalization performance.
    }%
    \label{fig:mnist-convergence}
\end{figure*}


Now we present experiments on the performance of sub-level set teleportation in
practice.
We use \cref{alg:teleport} to implement a teleportation oracle and evaluate the
merits of teleportation for accelerating deterministic and stochastic
optimization.

\textbf{Solving the Teleportation Problem}:
\cref{fig:toy-teleport} shows the convergence path of our teleportation solver
on two test functions: the Booth function and the Goldstein-Price function
\citep{goldstein1971descent}.
The Booth function is a convex quadratic, while Goldstein-Price is highly
non-convex; our solver converges to the global optimum of the teleportation
problem for the Booth function and a local maximum for the Goldstein-Price
function.
To demonstrate the scalability of our approach, we also solve
sub-level set teleportation for a two-layer MLP with fifty hidden
units and soft-plus activations on MNIST \citep{lecun1998mnist}.
We use weight-decay regularization to ensure the objective is coercive.
\cref{fig:tele-metrics} shows
the norm of the network gradient (i.e.\ the
teleportation objective), the KKT residual~\eqref{eq:kkt-conditions},
and violation of the constraints during teleportation.
Our algorithm converges to a KKT point where the gradient norm is two orders of
magnitude larger than that at the initialization.

\textbf{Performance Profile}:
Now we consider using teleportation to accelerate training of three-layer ReLU
networks using gradient methods.
We focus on iteration complexity since no previous paper has clearly
demonstrated that teleportation improves convergence even using a cost-free
teleportation oracle.
We generate \( 120 \) problems by trying six regularization strengths on \( 20
\) binary classification datasets from the UCI repository
\citep{asuncion2007uci}.
\cref{fig:network-profile} presents a performance profile
\citep{dolan2002benchmarking} comparing GD (SGD), GD with the Polyak step-size
(SPS) \citep{polyak1987introduction, loizou2021sps}, GD with line-search
(GD-LS), and normalized GD (Norm. SGD) with (solid lines) and without (dashed
lines) teleportation.

We find that access to a teleportation oracle almost uniformly improves
iteration complexity of both stochastic and deterministic optimizers
(nearly every solid line is above the corresponding dashed line).
Teleportation seems particularly useful for computing high-accuracy
solutions, as the gaps between methods with and without teleportation widen
--- meaning teleportation methods solve more problems to a high
accuracy --- when many iterations are used.
This matches our theory, which suggests that teleportation is most useful when
the optimality gap is small (see \cref{thm:stability-teleport-ls}).
We refer to \cref{fig:logistic-profile} for similar trends on convex problems
satisfying Hessian stability and Figures~\ref{fig:uci-convergence-logreg}
and~\ref{fig:uci-convergence-network} for convergence curves on specific
problems.

\begin{figure*}
    \centering
    \includegraphics[width=0.98\textwidth]{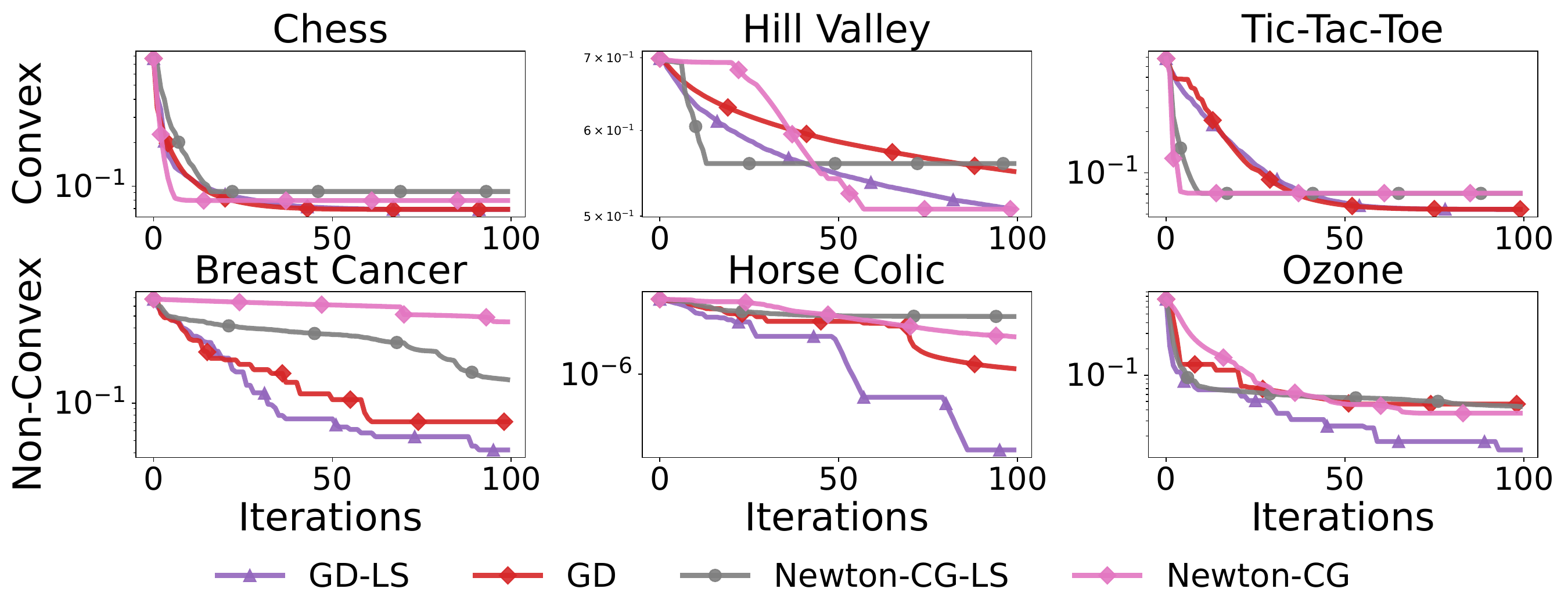}
    \caption{GD with teleportation compared to truncated Newton (Newton-CG),
        with and without line-search.
        Newton-CG is somewhat competitive on convex problems, but completely
        fails for non-convex optimization.
    }%
    \label{fig:newton-comparison}
    \vspace{-2ex}
\end{figure*}

\textbf{Time Complexity}:
Now we compare the wall-clock time of gradient methods with and without
teleportation.
Since the accuracy of the teleportation solver affects both the time complexity
of teleporting and the effectiveness of the teleportation steps, we leverage
this experiment to perform an ablation on the maximum number of iterations used
by \cref{alg:teleport}.
Figures~\ref{fig:logistic-profile-time} and~\ref{fig:network-profile-time}
show performance profiles on our 120 problems from the UCI repository
with total time on the \( x \)-axis.

Although solving the teleportation problem to high accuracy introduces
overhead, stochastic gradient methods using \( t = 50 \) iterations for the teleportation
sub-procedure solve significantly more problems than their standard
counterparts with the same computational budget.
This gap is particularly noticeable when training non-convex ReLU networks.
In comparison, teleportation methods with insufficiently many iterations (\( t
\in \cbr{1, 10} \)) are often slower than methods without teleportation without
noticeable improvements in accuracy.
We conclude that solving the teleportation problem accurately is critical and
can be computationally worthwhile when highly accurate solutions are desired.

\textbf{Image Classification}:
To confirm our observations, we train a two-layer network with \( 500 \) neurons
and soft-plus activations on MNIST.\
We use full-batch gradients and compare teleporting once every
five iterations, starting at \( k = 5 \), against the same methods without
teleportation.
\cref{fig:mnist-convergence} shows that, as on the UCI datasets, teleporting
leads to faster convergence for every optimization method considered.
This gap is particularly pronounced for normalized GD and GD with line-search.
Although teleportation steps significantly increase the gradient norm,
it decreases quickly and generalization performance is unaffected.
See Figures~\ref{fig:mnist-stochastic} and~\ref{fig:mnist-deterministic-non-smooth}
for results with stochastic gradients and non-smooth activations.


\textbf{Comparison with Newton}:
\cref{eq:newton-direction} implies that GD with teleportation can be viewed as
a type of Hessian-free Newton method.
We study this connection and compare GD with teleportation at every step
against Newton-CG  \citep{dembo1983truncated}, an inexact Newton update where
the search direction is computed using the conjugate gradient (CG) method.
Because both methods use only Hessian-vector products, we
can ensure an equal computational burden by fixing the number of
CG iterations to be the same as the number of iterations used by \cref{alg:teleport}.

\cref{fig:newton-comparison} reports training performance on three convex
logistic regression problems and three non-convex ReLU network training
problems.
We plot the minimum training objective attained so far at each iteration since
GD with teleportation at every iteration produces noisy iterates.
While Newton-CG is initially faster for 2/3 convex problems, it stalls due to
poor quality of the approximate search direction.
Moreover, Newton-CG fails on all three non-convex problems, likely because
it is attracted to any critical point, including saddles and local maxima.
This replicates well-known failure modes of Newton-type methods for non-convex
optimization \citep{xu2020newton}.
In contrast, GD with teleportation still obtains a fast linear rate of
convergence.
See Figures~\ref{fig:newton-comparison-logreg}
and~\ref{fig:newton-comparison-network} for figures with additional metrics.

\textbf{Additional Experiments}:
In \cref{app:additional-experiments} we present additional experiments on
(i) using sub-level set teleportation to initialize optimization and (ii) the
effects of regularization on the success of teleportation for non-coercive
ReLU networks.
Figures~\ref{fig:logistic-profile-init} and~\ref{fig:network-profile-init} show
that initializing using teleportation has mixed performance compared to
the standard Kaiming He initialization \citep{he2015initialization} and is 
particularly poor for logistic regression problems.
We speculate that this is because \( R^2 \) may be very large at the start of
optimization.
Surprisingly, \cref{fig:regularization-ablation} shows that there is no clear
connection between regularization and improved success rates for methods with
teleportation.


\section{CONCLUSION}\label{sec:conclusion}


Despite work advocating for level set teleportation, little has been
done to solve teleportation problems or evaluate their practical utility.
We rectify this by studying teleportation in detail;
we prove new theoretical guarantees for GD with teleportation, showing it can
obtain a combined linear/sub-linear rate, which improves significantly over
standard GD, when Hessian stability holds.
We also derive a novel algorithm for solving teleportation problems and
evaluate the performance of teleportation on a large suite of problems.
Our results reveal that access to a teleportation oracle implemented by our
algorithm speeds-up both convex and non-convex optimization can even outperform
approximate Newton methods.
Going beyond this oracle model to obtain full iteration complexities for GD
with an approximate teleportation method is an open problem that we leave to
future work.

\subsubsection*{Acknowledgements}

Aaron Mishkin was supported by NSF Grant DGE-1656518, by NSERC Grant
PGSD3-547242-2020, and by an internship at the Center for Computational
Mathematics, Flatiron Institute.
We thank Si Yi (Cathy) Meng and Michael Eickenberg for helpful conversations
while preparing this work and David Martínez-Rubio and Gauthier Gidel for
insightful comments on an earlier version.
We also thank the Scientific Computing Core, Flatiron Institute, for
supporting the numerical experiments in this work and Fabian Schaipp for use
of the \texttt{step-back} code.
We credit Ryan D'Orazio for drawing our attention to sufficient conditions for
coercivity and the anonymous reviewers for their invaluable comments on earlier
versions of this paper.

\printbibliography[]

\clearpage
\newpage

\section*{Checklist}


\begin{enumerate}

    \item For all models and algorithms presented, check if you include:
          \begin{enumerate}
              \item A clear description of the mathematical setting, assumptions, algorithm, and/or model. \textbf{Yes}
              \item An analysis of the properties and complexity (time, space, sample size) of any algorithm. \textbf{Yes}
              \item (Optional) Anonymized source code, with specification of all dependencies, including external libraries. \textbf{No}
          \end{enumerate}

          \textbf{Justification}: Source code will be released after acceptance.
          All algorithms are carefully described and illustrated with pseudocode.
          We prove iteration complexities for GD with teleportation and state when
          no iteration complexity is known (i.e.\ convergence is only known
          asymptotically).

    \item For any theoretical claim, check if you include:
          \begin{enumerate}
              \item Statements of the full set of assumptions of all theoretical results. \textbf{Yes}
              \item Complete proofs of all theoretical results. \textbf{Yes}
              \item Clear explanations of any assumptions. \textbf{Yes}
          \end{enumerate}

          \textbf{Justification}: Every theoretical result is presented in a formal
          statement with the complete set of necessary assumptions. All proofs are
          provided in the appendix.

    \item For all figures and tables that present empirical results, check if you include:
          \begin{enumerate}
              \item The code, data, and instructions needed to reproduce the main experimental results (either in the supplemental material or as a URL). \textbf{No}
              \item All the training details (e.g., data splits, hyperparameters, how they were chosen). \textbf{Yes}
              \item A clear definition of the specific measure or statistics and error bars (e.g., with respect to the random seed after running experiments multiple times). \textbf{Yes}
              \item A description of the computing infrastructure used. (e.g., type of GPUs, internal cluster, or cloud provider). \textbf{Yes}
          \end{enumerate}

          \textbf{Justification}: \cref{app:experiment-details} contains all
          hyper-parameter settings and algorithmic details necessary to
          reproduce our experiments. Algorithms are exactly described using
          pseudo-code.  Moreover, we will release the code to exactly reproduce
          our experiments upon publication.

    \item If you are using existing assets (e.g., code, data, models) or curating/releasing new assets, check if you include:
          \begin{enumerate}
              \item Citations of the creator If your work uses existing assets. \textbf{Yes}
              \item The license information of the assets, if applicable. \textbf{Yes}
              \item New assets either in the supplemental material or as a URL, if applicable. \textbf{Not Applicable}
              \item Information about consent from data providers/curators. \textbf{Not Applicable}
              \item Discussion of sensible content if applicable, e.g., personally identifiable information or offensive content. \textbf{Not Applicable} 
          \end{enumerate}

          \textbf{Justification}: We do not create or release new assets. All
          existing datasets and models have been carefully referenced in the
          text.

    \item If you used crowdsourcing or conducted research with human subjects, check if you include:
          \begin{enumerate}
              \item The full text of instructions given to participants and screenshots. \textbf{Not Applicable}
              \item Descriptions of potential participant risks, with links to Institutional Review Board (IRB) approvals if applicable. \textbf{Not Applicable}
              \item The estimated hourly wage paid to participants and the total amount spent on participant compensation. \textbf{Not Applicable}
          \end{enumerate}

          \textbf{Justification}: We did not use crowdsourcing.

\end{enumerate}

\onecolumn
\newpage
\appendix

\section{LEVEL SET TELEPORTATION: PROOFS}\label{app:teleportation}


\begin{definition}\label{def:non-expansive}
    An operator \( M : \R^d \into \R^d \) is called non-expansive
    if for every \( x, y \in \R^d \),
    \[
        \norm{M(x) - M(y)}_2 \leq \norm{x - y}_2.
    \]
\end{definition}

\begin{lemma}\label{lemma:relaxation-solutions}
    Suppose \( f \) is (strictly) convex.
    Then at least one (every) solution to the sub-level set teleportation
    problem (\cref{eq:sublevel-teleport}) is a solution to the level set
    teleportation problem (\cref{eq:teleport}).
\end{lemma}
\begin{proof}
    Let \( w(t) = w + t \nabla f(w) \).
    From convexity,
    \begin{align*}
        f(w(t)) - f(w)
                                         & \geq t \norm{\nabla f(w)}_2^2               \\
        f(w) - f(w(t))
                                         & \geq - t \abr{\nabla f(w(t)), \nabla f(w)}.
        \intertext{Adding these inequalities and using Cauchy-Schwarz,}
        \implies \norm{\nabla f(w(t))}_2 & \geq \norm{\nabla f(w)}_2.
    \end{align*}
    That is, the gradient norm is monotone non-decreasing
    along the gradient direction when \( f \) is convex.
    Thus, at least one solution to the maximization problem must occur on the boundary
    of the sub-level set, which completes the first part of the proof.
    For the second, simply note that the inequalities hold strictly if
    \( f \) is strictly convex, implying that every solution must be on the
    boundary.
\end{proof}

\diameterBound*
\begin{proof}
    Starting from the descent lemma (\cref{eq:descent-lemma}),
    we obtain
    \begin{align*}
        f(\wkk)
         & \leq f(\wk^+) + \etak \rbr{\frac{L \etak}{2} - 1}\norm{\grad(\wk^+)}_2^2 \\
         & \leq f(\wk) + \etak \rbr{\frac{L \etak}{2} - 1}\norm{\grad(\wk)}_2^2     \\
         & \leq f(\wk),
    \end{align*}
    since \( \etak < 2 / L \).
    Subtracting \( f(\wopt) \) from both sides implies
    \( \delta_{k+1} \leq \delta_k \).
    Thus, \( \wkk \in \calS_0 \) and \( \norm{\wkk - \wopt}_2 \leq R \) by
    definition.
    Arguing by induction now completes the proof.
\end{proof}

\scProgress*
\begin{proof}
    Since \( f \) is \( L \)-smooth and \( \mu \) strongly convex, \( \nabla f \)
    satisfies the following inequality:
    \begin{equation}\label{eq:coercivity-gradient}
        \abr{\grad(x) - \grad(y), x - y}
        \geq \frac{\mu L}{\mu + L}\norm{x - y}_2^2
        + \frac{1}{L + \mu}\norm{\grad(x) - \grad(y)}_2^2.
    \end{equation}
    This is sometimes called coercivity of the gradient; see, for example,
    \citet[Theorem 2.1.12]{nesterov2004lectures}.
    Suppose \( k \in \calT \).
    Then,
    \begin{align*}
        \norm{\wkk - \wopt}_2^2
         & = \norm{\wk^+ - \etak \grad(\wk^+) - \wopt}_2^2                                                             \\
         & = \norm{\wk^+ - \wopt}_2^2 - 2 \etak \abr{\grad(\wk^+), \wk^+ - \wopt} + \etak^2 \norm{\grad(\wk^+)}_2^2    \\
         & \leq \norm{\wk^+ - \wopt}_2^2 - 2 \etak \rbr{\frac{\mu L}{\mu + L}\norm{\wk^+ - \wopt}_2^2
        + \frac{1}{L + \mu}\norm{\grad(\wk^+)}_2^2}                                                                    \\
         & \hspace{3em} + \etak^2 \norm{\grad(\wk^+)}_2^2                                                              \\
         & = \rbr{1 - \frac{2 \etak \mu L}{\mu + L}}\norm{\wk^+ - \wopt}_2^2
        + \etak \rbr{\etak - \frac{2}{\mu + L}}\norm{\grad(\wk^+)}_2^2                                                 \\
         & \leq \rbr{1 - \frac{2 \etak \mu L}{\mu + L}}\norm{\wk^+ - \wopt}_2^2                                        \\
         & \hspace{3em} + \etak \max\cbr{L^2 \rbr{\etak - \frac{2 }{\mu + L}}, \mu^2 \rbr{\etak - \frac{2 }{\mu + L}}}
        \norm{\wk^+ - \wopt}_2^2                                                                                       \\
         & = \max\cbr{(1 - \etak L)^2, (1 - \etak \mu)^2} \norm{\wk^+ - \wopt}_2^2,
    \end{align*}
    where we have used
    \( \mu \) strong convexity and \( L \)-smoothness to bound
    the gradient norm depending on the step-size \( \etak \).
    This completes the proof.
\end{proof}

\scTeleport*
\begin{proof}
    First we show the upper bound when teleportation is not non-expansive.
    Starting from the descent lemma \cref{eq:descent-lemma}, we have
    \begin{align*}
        f(\wkk) - f(\wopt)
         & \leq f(\wk^+) - f(\wopt) - \etak \rbr{1 - \frac{\etak L}{2}} \norm{\grad(\wk^+)}_2^2     \\
         & \leq f(\wk) - f(\wopt) - \etak \rbr{1 - \frac{\etak L}{2}} \norm{\grad(\wk)}_2^2         \\
         & \leq f(\wk) - f(\wopt) - 2 \mu \etak \rbr{1 - \frac{\etak L}{2}} \rbr{f(\wk) - f(\wopt)} \\
         & = \rbr{1 - 2 \mu \etak \rbr{1 - \frac{\etak L}{2}}} \rbr{f(\wk) - f(\wopt)},
    \end{align*}
    where in the last inequality we used \( \etak < \frac{2}{L} \)
    and the PL-condition,
    \[
        \frac{1}{2 \mu} \norm{\grad(w)}_2^2 \geq f(w) - f(\wopt),
    \]
    which is implied by strong convexity.
    Recursing on the final inequality implies
    \[
        \delta_{K}
        \leq
        \prod_{i=0}^{K-1} \rbr{1 - 2 \mu \eta_i \rbr{1 - \frac{\eta_i L}{2}}} \delta_0,
    \]
    as claimed.

    Now we consider the setting where teleportation is assumed to be non-expansive.
    Starting from \cref{lemma:sc-progress}, we obtain,
    \begin{align*}
        \norm{\wkk - \wopt}_2^2
         & \leq \max\cbr{(1 - \etak L)^2, (1 - \etak \mu)^2} \norm{\wk^+ - \wopt}_2^2                    \\
         & \leq \max\cbr{(1 - \etak L)^2, (1 - \etak \mu)^2} \norm{\wk - \wopt}_2^2,
        \intertext{by non-expansivity of teleportation and the fact that
            \( \wopt \) is a fixed point of the teleportation operator.
            Recursing on this inequality implies,
        }
        \norm{w_K - \wopt}_2^2
         & \leq \prod_{i=0}^{K-1} \max\cbr{(1 - \eta_i L)^2, (1 - \eta_i \mu)^2} \norm{\w_0 - \wopt}_2^2 \\
        \implies
        f(\w_K) - f(\wopt)
         & \leq \frac{L}{\mu} \sbr{\prod_{i=0}^{K-1} \max\cbr{(1 - \eta_i L)^2, (1 - \eta_i \mu)^2}}
        \rbr{f(w_0) - f(\wopt)},
    \end{align*}
    where the last inequality uses both smoothness and strong convexity to
    upper and lower bound \( \norm{w - \wopt}_2^2 \) in terms of
    \( f(w) - f(\wopt) \).

    For the lower bound, suppose \( f(w) = \frac{1}{2} \norm{w}_2^2 \).
    Clearly \( f \) is \( 1 \)-smooth and strongly convex with \( \mu = 1 \).
    The unique minimizer is simply \( \wopt = 0 \).
    Starting from an arbitrary \( w_0 \), each iteration of GD with teleportation
    has the following recursion:
    \[
        \wkk = \wk^+ - \etak \wk^+ = (1 - \etak ) \wk^+,
    \]
    and the objective evolves trivially as
    \begin{align*}
        \frac{1}{2} \norm{\wkk}_2^2
         & = (1 - \etak)^2 \half \norm{\wk^+}_2^2 \\
         & = (1 - \etak)^2 \half \norm{\wk}_2^2,
    \end{align*}
    where the second equality uses \cref{lemma:relaxation-solutions} to guarantee
    that the solution to sub-level set teleportation lies on the level
    set \( \calL_k \) and the fact that every point \( x \in \calL_k \)
    satisfies \( \norm{x}_2 = \norm{\wk} \).
    Thus, each step of gradient descent makes progress exactly
    matching the convergence rate for general smooth, strongly convex functions
    regardless of teleportation.
    This completes the lower-bound.
\end{proof}

\subsection{Convergence for Convex Functions: Proofs}\label{app:convex-convergence}

\newtonDirection*
\begin{proof}
    Define the second-order Huber function,
    \begin{equation}
        h_2(x) =
        \begin{cases}
            -\frac{1}{8} x^4 + \frac{3}{4} x^2 + \frac{3}{8} & \mbox{if \( \abs{x} < 1 \),} \\
            \abs{x}                                          & \mbox{Otherwise.}
        \end{cases}
    \end{equation}
    It is straightforward to verify that \( h_2 \) is \( C^2 \) and convex with
    gradient function
    \[
        \nabla h_2(x) =
        \begin{cases}
            -\frac{1}{2} x^3 + \frac{3}{2} x & \mbox{if \(\abs{x} < 1\),} \\
            \text{sign}(x)                   & \mbox{otherwise.}
        \end{cases}
    \]
    Consider the family of target functions \( f_d : \R^d \into \R \) given by
    \[
        f_d(x) = \sum_{i = 1}^d h_2(x_i),
    \]
    where \( d \) is an integer.
    Given level \( b \geq 1 \), define \( q = \floor{b} \) and
    consider the teleportation problem,
    \[
        x^+
        \in \argmax_x \cbr{\norm{\grad_{q}(x)}_2^2 : f_{q}(x) = b}.
    \]
    Since \( f_q \) is symmetric, we restrict ourselves to maximization over
    the non-negative orthant.
    At any point \( x \), the squared-norm of the gradient is given by
    \[
        \norm{\nabla f_q(x)}_2^2 = \sum_{i=1}^q (\nabla h_2(x_i))^2,
    \]
    which is maximized when \( x_i \geq 1 \) for each \( i \).
    Since \( q < b \), the solution set of the teleportation
    problem is the following:
    \[
        \calX^* = \cbr{x : x_i \geq 1, \sum_{i = 1}^q x_i = b}.
    \]
    At any of these points, the gradient is simply \( \nabla f_n(x) = \mathbf{1} \).
    However, the Hessian is \( \nabla^2 f(x) = 0 \), meaning
    the gradient does not coincide with the Newton direction, which is not defined.

    Although this argument requires \( b \geq 1 \), in practice any \( b > 0 \)
    can be handled by adjusting the point at which \( h_2 \) transitions from a
    linear to a quartic function.
\end{proof}

\convexTeleport*
\begin{proof}
    The proof follows \citet[Fact B.1]{zhu2017linear} and
    \citet[Theorem 3.3]{bubeck2015convex}.
    Starting from \( L \)-smoothness of \( f \), we obtain
    \begin{align}
        f(\wkk)
         & \leq f(\wk^+) + \eta \rbr{\frac{L \eta}{2} - 1}\norm{\grad(\wk^+)}_2^2 \nonumber \\
         & \leq f(\wk) + \eta \rbr{\frac{L \eta}{2} - 1}\norm{\grad(\wk)}_2^2,              %
        \label{eq:teleport-descent}
    \end{align}
    which implies that \( \cbr{f(\wk)} \) are monotone decreasing regardless
    of the teleportation schedule.
    Thus, \( \norm{\wkk - \wopt}_2 \leq R \) for all \( k \) and we
    use convexity of \( f \) and the Cauchy-Schwarz inequality to obtain,
    \begin{align*}
        f(\wk) - f(\wopt)
         & \leq \abr{\grad(\wk), \wk - \wopt} \\
         & \leq R \norm{\grad(\wk)}_2,
    \end{align*}
    which holds for all \( k \).
    Let \( \delta_k = f(\wk) - f(\wopt) \).
    Using \cref{eq:teleport-descent} from above and then
    dividing on both sides by \( \delta_k \) and \( \delta_{k+1} \) gives
    the following:
    \begin{align*}
        \delta_{k}^2
         & \leq \frac{2R^2}{\eta\rbr{2 - L \eta}} \rbr{\delta_k - \delta_{k+1}}             \\
        \implies
        1 \leq \frac{\delta_k}{\delta_{k+1}}
         & \leq \frac{2R^2}{\eta\rbr{2 - L \eta}} \rbr{\inv{\delta_{k+1}} - \inv{\delta_k}} \\
        \implies
        \frac{\eta\rbr{2 - L \eta}}{2 R^2}
         & \leq \rbr{\inv{\delta_{k+1}} - \inv{\delta_k}},
        \intertext{so that summing from \( k = 0 \) to \( K - 1 \) gives}
        \frac{K \eta\rbr{2 - L \eta}}{2 R^2}
         & \leq \rbr{\inv{\delta_{K}} - \inv{\delta_0}}                                     \\
         & \leq \inv{\delta_{K}},
    \end{align*}
    which completes the first part of the proof.

    For the second part of the proof, consider the second-order Huber function,
    \begin{equation}
        h_2(x) =
        \begin{cases}
            -\frac{1}{8} x^4 + \frac{3}{4} x^2 + \frac{3}{8} & \mbox{if \( \abs{x} < 1 \),} \\
            \abs{x}                                          & \mbox{Otherwise.}
        \end{cases}
    \end{equation}
    It is straightforward to verify that \( h_2 \) is \( C^2 \) and convex.
    Let \( f(w) = h_2(w_1) \), meaning \( f \) is constant in every
    dimension but the first.
    This implies that the teleportation operator
    \[
        \wk^+ = \argmax \cbr{\half \norm{\grad(\w)}_2^2 : f(w) \leq f(\wk)},
    \]
    is simply the identity operator, i.e. \( \wk^+ = \wk \).
    Thus, the optimization paths of vanilla GD and GD with teleportation
    are identical and the two methods have the same convergence rate.
\end{proof}

\teleportDistances*
\begin{proof}
    We assume for convenience that \( w = (x, y) \),
    \( x > 1 \) and \( 0 < y < \alpha \).
    These assumptions can be relaxed by modifying our construction.
    Let \( \epsilon > 0 \) such that \( \alpha, y > \epsilon \) and
    \( x + y \epsilon > \alpha \epsilon + 1 \) hold, where both conditions can
    be satisfied by taking \( \epsilon \) sufficiently small.

    Let \( g_\delta \) be the Huber function defined by
    \begin{equation}\label{eq:huber-function}
        g_\delta(x) =
        \begin{cases}
            \half x^2                         & \mbox{if \( x \leq \delta \)} \\
            \delta \rbr{\abs{x} - \delta / 2} & \mbox{otherwise.}
        \end{cases}
    \end{equation}
    and consider the objective function
    \begin{equation}
        f_{(\epsilon, \alpha)}(x, y) = g_1(x) + g_\epsilon(y)
        + \half \mathds{1}_{y \geq \alpha} \rbr{y - \alpha}^2.
    \end{equation}
    We first show that there exists \( x' = 1 \) and some \( y' > \alpha \) satisfying
    \[
        f_{(\epsilon, \alpha)}(x', y') = f_{(\epsilon, \alpha)}(x, y).
    \]
    For this to hold, we must have
    \begin{align*}
             & f_{(\epsilon, \alpha)}(x', y') = \half + \epsilon y' - \frac{\epsilon^2}{2} + \half \rbr{y' - \alpha}^2
        = x - \half + \epsilon y - \frac{\epsilon^2}{2} = f_{(\epsilon, \alpha)}(x, y)                                 \\
        \iff & (y')^2 + 2 \rbr{\epsilon - \alpha} y' + \rbr{\alpha^2 - 2x - 2 \epsilon y + 2} = 0                      \\
        \iff & y' = \alpha - \epsilon \pm \sqrt{\rbr{\epsilon - \alpha}^2 - \alpha^2 + 2x + 2 \epsilon y - 2}.
        \intertext{In particular, for \( y' > \alpha \), it must hold that}
             & \rbr{\epsilon - \alpha}^2 - \alpha^2 + 2x + 2 \epsilon y - 2 > \epsilon^2                               \\
             & \iff x + \epsilon y > \alpha \epsilon + 1,
    \end{align*}
    where this last condition is guaranteed by assumption on \( \epsilon \).
    We conclude that \( (x', y') \) is on the level set as desired.

    The gradient of \( f_{(\epsilon, \alpha)} \) is easy to calculate as
    \[
        \nabla f_{(\epsilon, \alpha)}(x, y) =
        \begin{cases}
            \sign(x) + \epsilon \cdot \sign(y)
            + \mathds{1}_{y \geq \alpha} \rbr{y - \alpha}
             & \mbox{if \( \abs{x} \geq 1, \abs{y} \geq \epsilon \)}  \\
            x + \epsilon \cdot \sign(y)
            + \mathds{1}_{y \geq \alpha} \rbr{y - \alpha}
             & \mbox{if \( \abs{x} \leq 1, \abs{y} \geq \epsilon \)}  \\
            \sign(x) + \epsilon y
            + \mathds{1}_{y \geq \alpha} \rbr{y - \alpha}
             & \mbox{if \( \abs{x} \geq 1, \abs{y} \leq \epsilon \)}  \\
            x + \epsilon y
            + \mathds{1}_{y \geq \alpha} \rbr{y - \alpha}
             & \mbox{if \( \abs{x} \leq 1, \abs{y} \leq \epsilon \)}.
        \end{cases}
    \]
    In particular, the gradient norm at \( (x', y') \) is given by
    \[
        \norm{\nabla f_{(\epsilon, \alpha)}(x', y')}_2^2
        = 1 + \epsilon^2 + (y' - \alpha)^2.
    \]
    A straightforward case analysis reveals that for every \( \bar x \),
    \( \bar y \) such that \( \bar y < \alpha \),
    \[
        \norm{\nabla f_{(\epsilon, \alpha)}(\bar x, \bar y)}_2^2
        < \norm{\nabla f_{(\epsilon, \alpha)}(x', y')}_2^2.
    \]
    That is, the maximizer of the gradient norm on the level set
    must satisfy \( y^+ \geq \alpha \).
    We conclude that
    \[
        \norm{w^+ - \wopt}_2^2 = \norm{(x^+, y^+)}_2^2 \geq \alpha^2.
    \]

    Now we upper-bound the diameter of the sub-level set as
    \begin{align*}
        R^2
         & = \max \cbr{\norm{(x', y')}_2
        : f_{(\epsilon, \alpha)} ((x', y')) = f_{(\epsilon, \alpha)} ((x, y))} \\
         & \leq \max \cbr{|x'| + |y'|
            : \exists \tilde x, \tilde y \text{ s.t. }
            f_{(\epsilon, \alpha)} ((x', \tilde y)) = f_{(\epsilon, \alpha)}((x, y)),
            f_{(\epsilon, \alpha)} ((\tilde x, y')) = f_{(\epsilon, \alpha)} ((x, y))}.
    \end{align*}
    Taking \( \tilde x = \tilde y = 0 \) allow us to maximize \( x', y' \).
    As a result, the constraint on \( x' \) implies
    \[
        x' - \half = x - \half + \epsilon (y - \epsilon / 2),
    \]
    from which we deduce
    \[
        x' = x + \epsilon y - \epsilon^2 / 2 \leq x + \epsilon y.
    \]
    Repeating a similar calculation for \( y' \), we obtain the constraint
    \begin{align*}
        \epsilon (y' - \epsilon / 2) + \half (y' - \alpha)^2
         & = x + \epsilon(y - \epsilon / 2)                                                 \\
        \implies y'
         & = \alpha - \epsilon \pm \sqrt{\epsilon^2 - 2\alpha \epsilon + 2x + 2 \epsilon y} \\
        \implies (y' - \alpha + \epsilon)^2
         & = \epsilon^2 - 2 \alpha \epsilon + 2 x + 2 \epsilon y                            \\
        \implies (y')^2
         & = 2 \alpha y' - 2(y' - y) \epsilon + 2x - \alpha^2                               \\
        \implies y'
         & \leq 2 \alpha + \frac{2x}{\alpha},
    \end{align*}
    where we have used \( y' \geq \alpha > y \).
    Combining these results, we have
    \begin{align*}
        R
         & \leq x + \epsilon y + 2 \alpha + \frac{2x}{\alpha}.
        \intertext{Taking \( \alpha \gg x \), \( x \) arbitrarily close to \( 1 \),
            and \( \epsilon \) sufficiently small implies there exists
            an absolute constant \( C < 2 \) such that}
        R
         & \leq 2 \alpha + C < 4\alpha.
    \end{align*}
    Recalling that \( \norm{(x^+, y^+)}_2 \geq \alpha \) shows that
    teleportation comes within a small constant factor of the diameter of the
    sub-level set.
\end{proof}

\subsection{Convergence under Hessian Stability: Proofs}\label{app:stability-convergence}

\begin{lemma}\label{lemma:stability-bounds}
    If \( f \) satisfies \( (\tL, \tmu) \)-Hessian stability, then
    minimizing on both sides of \cref{eq:stability-lower-bound} implies
    \begin{equation}\label{eq:stability-pl}
        f(\xopt) \geq f(x) - \frac{c}{2} \norm{\grad(x)}_{\nabla^2 f(x)^{-1}}.
    \end{equation}
\end{lemma}
\begin{proof}
    See \citet[Lemma 2]{karimireddy2019newton} for proof.
\end{proof}

\stabilityTeleportProgress*
\begin{proof}
    \cref{lemma:diameter-bound} implies that Hessian stability holds
    since the iterates remain inside the initial sub-level set.
    Starting from \cref{eq:stability-upper-bound}, we obtain
    \begin{align*}
        f(\wkk)
         & \leq f(\wk^+) + \abr{\grad(\wk^+), \wkk - \wk^+}
        + \frac{\tL}{2} \norm{\wkk - \wk^+}^2_{\nabla^2 f(\wk^+)} \\
         & \leq f(\wk)
        - \lambda_k \eta \norm{\grad(\wk^+)}^2_{\nabla^2 f(\wk^+)^{-1}}
        + \frac{\tL \lambda_k^2 \eta^2}{2}
        \norm{\grad(\wk^+)}^2_{\nabla^2 f(\wk^+)^{-1}}            \\
         & = f(\wk)
        + \lambda_k \eta \rbr{\frac{\tL \lambda_k \eta}{2} - 1}
        \norm{\grad(\wk^+)}^2_{\nabla^2 f(\wk^+)^{-1}}            \\
         & \leq f(\wk)
        + 2\lambda_k \eta \tmu \rbr{\frac{\tL \lambda_k \eta}{2} - 1}
        \sbr{f(\wk^+) - f(\wopt)}                                 \\
         & = f(\wk)
        + 2\lambda_k \eta \tmu \rbr{\frac{\tL \lambda_k \eta}{2} - 1}
        \sbr{f(\wk) - f(\wopt)},
    \end{align*}
    where we have used \( \lambda_k \leq L \) to guarantee
    \( \frac{\tL \lambda_k \eta}{2} < 1 \) and then applied \cref{eq:stability-pl}.
    The final equality follows since Hessian stability implies \( f \) is
    strictly convex, at which point \cref{lemma:relaxation-solutions} guarantees
    \( f(\wk^+) = f(\wk) \).
    Adding and subtracting \( f(\wopt) \) on both sides yields the final result,
    \begin{equation*}
        f(\wkk) - f(\wopt) \leq
        \rbr{1 + 2\lambda_k \eta \tmu \rbr{\frac{\tL \lambda_k \eta}{2} - 1}}
        \rbr{f(\wk) - f(\wopt)}.
    \end{equation*}
\end{proof}

\stabilityTeleportProgressLS*
\begin{proof}
    Satisfaction of the line-search criterion in \cref{eq:armijo-ls}
    implies \( f(\wkk) \leq f(\wk^+) = f(\wk) \),
    meaning the first GD step after teleporting can only decrease the
    function value and thus \( \wkk \in \calS_0 \).
    So, \( \wk, \wk^+, \) and \( \wkk \) satisfy stability of the Hessian.

    Now we use stability to control the step-size \( \eta \).
    Using \( \grad(\wk^+) = \lambda_k \nabla^2 f(\wk^+)^{-1} \grad(\wk^+) \)
    from teleportation optimality conditions, the line-search criterion
    is equivalent to,
    \[
        f(\wkk)
        \leq f(\wk^+)
        - \frac{\lambda_k \eta}{2}\norm{\grad(\wk^+)}_{\nabla^2 f(\wk^+)^{-1}}^2.
    \]
    At the same time, Hessian stability implies
    \begin{align*}
        f(\wkk)
         & \leq f(\wk^+) + \abr{\grad(\wk^+), \wkk - \wk^+}
        + \frac{\tL}{2} \norm{\wkk - \wk^+}^2_{\nabla^2 f(\wk^+)} \\
         & \leq f(\wk)
        - \lambda_k \eta \norm{\grad(\wk^+)}^2_{\nabla^2 f(\wk^+)^{-1}}
        + \frac{\tL \lambda_k^2 \eta^2}{2}
        \norm{\grad(\wk^+)}^2_{\nabla^2 f(\wk^+)^{-1}}            \\
         & = f(\wk)
        + \lambda_k \eta \rbr{\frac{\tL \lambda_k \eta}{2} - 1}
        \norm{\grad(\wk^+)}^2_{\nabla^2 f(\wk^+)^{-1}}.
    \end{align*}
    Since \( \eta \) is the largest step-size for which the line-search
    criterion holds, we can deduce the following bound on the step-size:
    \begin{align*}
        \frac{\lambda_k\eta}{2}
         & \geq \lambda_k \eta \rbr{1 - \frac{\tL \lambda_k \eta}{2}} \\
        \implies \eta
         & \geq \frac{1}{\tL \lambda_k}.
    \end{align*}
    Substituting this equation back into the line-search criterion and using
    the lower bound from Hessian stability (\cref{eq:stability-pl}) yields
    \begin{align*}
        f(\wkk)
         & \leq f(\wk^+) - \frac{1}{2 \tL}\norm{\grad(\wk^+)}_{\nabla^2 f(\wk^+)^{-1}}^2 \\
         & \leq f(\wk^+) - \frac{\tmu}{\tL} \rbr{f(\wk^+) - f(\wopt)}                    \\
         & = f(\wk) - \frac{\tmu}{\tL} \rbr{f(\wk) - f(\wopt)},
    \end{align*}
    where we used the fact that Hessian stability implies \( f \) is
    strictly convex, at which point \cref{lemma:relaxation-solutions} guarantees
    \( f(\wk^+) = f(\wk) \).
    Adding and subtracting \( f(\wopt) \) on both sides finishes the proof,
    \begin{equation*}
        f(\wkk) - f(\wopt) \leq
        \rbr{1 - \frac{\tmu}{\tL}}
        \rbr{f(\wk) - f(\wopt)}.
    \end{equation*}
\end{proof}

\stabilityTeleportLS*
\begin{proof}
    Satisfaction of the line-search criterion in \cref{eq:armijo-ls} implies \(
    f(\wkk) \leq f(\wk) \) regardless of teleportation.
    As a result, \( \cbr{f(\wk)} \) are non-increasing and Hessian stability
    holds for all iterates.

    Assume \( k \in \calB \).
    Using \cref{lemma:stability-teleport-progress-ls} and unrolling
    \cref{eq:stability-teleport-progress-ls} for the \( b_k \) teleportation
    steps, we obtain
    \begin{align*}
        \delta_{k + b_k}
         & \leq \rbr{1 - \frac{\tmu}{\tL}}^{b_k}
        \delta_{k}                               \\
         & \leq \rbr{1 - \frac{\tmu}{\tL}}^{b_k}
        \sbr{\delta_{k-1} - \frac{\eta_{k-1}}{2} \norm{\grad(w_{k-1})}_2^2},
    \end{align*}
    where the second inequality follows from the line-search.
    Since \( \eta \) is the largest step-size for which \cref{eq:armijo-ls}
    holds and \( f \) is \( L \)-smooth, the descent lemma (\cref{eq:descent-lemma})
    yields
    \[
        \frac{\eta_{k-1}}{2} \geq \eta_{k-1} \rbr{1 - \frac{\eta_{k-1} L}{2}}
        \implies \eta_{k-1} \geq \frac{1}{L}.
    \]
    Combining control of \( \eta_{k-1} \) with our bound on \( \delta_{k+b_k} \)
    yields,
    \[
        \delta_{k+b_k} \leq \rbr{1 - \frac{\tmu}{\tL}}^{b_k}
        \sbr{\delta_{k-1} - \frac{L}{2} \norm{\grad(w_k-1)}_2^2}.
    \]
    Re-arranging this inequality allows us to control the gradient norm
    as follows:
    \begin{equation}\label{eq:multi-step-descent-ls}
        \begin{aligned}
            \norm{\grad(w_{k-1})}_2^2
             & \leq 2 L \sbr{\delta_{k-1} - \delta_{k + b_k}} +
            2L \sbr{1 - \rbr{\frac{\tL}{\tL - \tmu}}^{b_k}}
            \delta_{k+b_k}
        \end{aligned}
    \end{equation}
    Starting now from the optimality gap and using convexity as well
    the Cauchy-Schwarz inequality, we obtain
    \begin{align*}
        \delta_{k-1}
         & \leq \abr{\grad(w_{k-1}), \w_{k-1} - \wopt} \\
         & \leq R \norm{\grad(\w_{k-1})}_2.
    \end{align*}
    Substituting in \cref{eq:multi-step-descent-ls} from above and then
    dividing on both sides by \( \delta_{k-1} \) and \( \delta_{k+b_k} \) gives
    the following:
    \begin{align}
        \delta_{k-1}^2
         & \leq 2R^2 L\rbr{\delta_{k-1} - \delta_{k+b_k}}
        + 2 R^2 L \sbr{1 - \rbr{\frac{\tL}{\tL - \tmu}}^{b_k}} \delta_{k+b_k}       \nonumber \\
        \implies
        1 \leq \frac{\delta_{k-1}}{\delta_{k+b_k}}
         & \leq 2 R^2 L \rbr{\inv{\delta_{k+b_k}} - \inv{\delta_{k-1}}}
        + 2 R^2 L \sbr{1 - \rbr{\frac{\tL}{\tL - \tmu}}^{b_k}} \inv{\delta_{k-1}} \nonumber   \\
        \implies
        \frac{1}{2 R^2 L}
         & \leq \inv{\delta_{k+b}} - \inv{\delta_{k-1}}
        + \sbr{1 - \rbr{\frac{\tL}{\tL - \tmu}}^{b_k}} \inv{\delta_{k-1}} \label{eq:stability-ls-recursion}.
    \end{align}
    Since \( k \in \calB \), it must be that \( k + b_k \not \in \calB \).
    Then, either we have \( k + b_k + 1 \in \calB \) and
    \cref{eq:stability-ls-recursion} holds with the choice of \( k' = k + b_k + 1 \),
    that is,
    \[
        \frac{1}{2 R^2 L}
        \leq \inv{\delta_{k + b_{k'} + b_k + 1}} - \inv{\delta_{k + b_k}}
        + \sbr{1 - \rbr{\frac{\tL}{\tL - \tmu}}^{b}} \inv{\delta_{k + b_k}},
    \]
    or \( k + b_k + 1 \) is also not a teleportation step
    and we obtain the following the simpler
    inequality (see \cref{prop:convex-teleport}):
    \[
        \frac{1}{2 R^2 L}
        \leq \inv{\delta_{k+b_k+1}} - \inv{\delta_{k+b_k}}.
    \]
    In either case, these two inequalities form a telescoping series
    with \cref{eq:stability-ls-recursion}.

    Note that this analysis only holds for \( k \in \calT \) such that
    \( k - 1 \) was a gradient step.
    In this special case where \( k = 0 \in \calT \),
    we must treat the update at \( k \) as a regular gradient step for
    \cref{eq:stability-ls-recursion} to hold.
    However, this causes no problems since every GD step after teleporting
    can also be analyzed as a vanilla GD step using the descent lemma.

    Summing over this telescoping series allows us to obtain
    the following:
    \begin{align*}
        \frac{K - |\calT|}{2 R^2 L}
         & \leq \rbr{\inv{\delta_{K}} - \inv{\delta_0}}
        + \sum_{k \in \calB} \sbr{1 - \rbr{\frac{\tL}{\tL - \tmu}}^{b_k}} \inv{\delta_{k-1}}
        \\
         & \leq \inv{\delta_{K}}
        + \sum_{k \in \calB} \sbr{1 - \rbr{\frac{\tL}{\tL - \tmu}}^{b_k}} \inv{\delta_{k-1}}.
    \end{align*}
    Finally, re-arranging this expression yields
    \[
        \delta_K \leq
        \frac{2 R^2 L}
        {K - |\calT| + 2 R^2 L \sum_{k \in \calB}
            \sbr{\rbr{\frac{\tL}{\tL - \tmu}}^{b_k} - 1} \inv{\delta_{k-1}}}.
    \]
\end{proof}

\stabilityCombinedLS*
\begin{proof}
    Define the constant,
    \[
        \zeta_{k-1} =
        \begin{cases}
            \rbr{\frac{\tL}{\tL - \tmu}}^{b_k}
              & \mbox{if \( k \in \calB \)} \\
            1 & \mbox{otherwise.}
        \end{cases}
    \]
    Using this with \cref{eq:stability-ls-recursion} implies
    \begin{align*}
        \frac{1}{2 R^2 L}
         & \leq \inv{\delta_{k+b_k}} - \inv{\delta_{k-1}}
        + \sbr{1 - \rbr{\frac{\tL}{\tL - \tmu}}^{b_k}} \inv{\delta_{k-1}} \\
         & = \inv{\delta_{k+b_k}} - \frac{\zeta_{k-1}}{\delta_{k-1}}.
    \end{align*}
    Since \( k \in \calB \), it must be that \( k + b_k \not \in \calB \).
    If \( k + b_k + 1 \in \calB \), then
    \cref{eq:stability-ls-recursion} also holds with the choice of \( k' = k + b_k + 1 \)
    and
    \[
        \frac{1}{2 R^2 L}
        \leq \inv{\delta_{k + b_{k'} + b_k + 1}} - \frac{\zeta_{k+b_k}}{\delta_{k + b_k}}.
    \]
    On the other hand, if \( k + b_k + 1 \) is also not a teleportation step, then
    we have (see \cref{prop:convex-teleport}):
    \[
        \frac{1}{2 R^2 L}
        \leq \inv{\delta_{k+b_k+1}} - \frac{\zeta_{k+b_k}}{\delta_{k+b_k}}.
    \]
    Let \( \theta_0 = 1 \) and \( \theta_{k+b_k} = \theta_{k-1} / \zeta_{k-1} \).
    We must choose the next \( \theta \) to do a weighted telescoping of these equations.
    If \( k + b_k + 1 \in \calT \), then set \( \theta_{k+b_{k'}+b_k+1} = \theta_{k + b_k} / \zeta_{k+b_k}  \).
    On the other hand, if \( k + b_k + 1 \not \in \calB \),
    then we must set \( \theta_{k+b_k+1} = \theta_{k+b_k} / \zeta_{k+b_k} \).
    In either case, weighting the telescoping series with the \( \theta_i \)'s
    and summing gives the following:
    \begin{align*}
        \frac{\theta_{k+b_k}}{2 R^2 L}
         & \leq \frac{\theta_{k+b_k}}{\delta_{k+b_k}} - \frac{\theta_{k-1}}{\delta_{k-1}} \\
        \implies
        \frac{\sum_{k \not \in \calT} \theta_{k}}{2 R^2 L}
         & \leq \frac{\theta_{K}}{\delta_{K}} - \frac{\theta_{0}}{\delta_{0}}             \\
         & \leq \frac{\theta_{K}}{\delta_{K}}.
    \end{align*}
    Re-arranging this equation gives
    \[
        \delta_K
        \leq \frac{2 R^2 L}{\sum_{k \not \in \calT, k > 0} \theta_{k} / \theta_{K}}.
    \]
    To determine the rate of convergence, we must compute
    \( \theta_{k} / \theta_{K} \) for each \( k \not \in \calT \).

    By construction, the weight \( \theta_{k} \) satisfies,
    \begin{align*}
        \theta_{k}
        = \prod_{i \not \in \calT, i \leq k} \zeta_{i-1}
        = \prod_{i \in \calB, i \leq k}
        \rbr{\frac{\tL - \tmu}{\tL}}^{b_k}
        = \rbr{\frac{\tL - \tmu}{\tL}}^{|\calT_i|},
    \end{align*}
    where \( \calT_i = \cbr{ i \in \calT : i \leq k} \).
    Let \( n_k = (|\calT| - |\calT_k|) \), which is the number of teleport
    steps  after iteration \( k \).
    In this notation, we have
    \( \theta_{k} / \theta_{K} = \rbr{\frac{\tL}{\tL - \tmu}}^{n_k}, \)
    meaning our final convergence rate is given by,
    \[
        \delta_K
        \leq \frac{2 R^2 L}{\sum_{k \not \in \calT, k > 0}
            \rbr{\frac{\tL}{\tL - \tmu}}^{n_k}}.
    \]
\end{proof}

\begin{corollary}\label{cor:alternating-steps-ls}
    Consider the setting of \cref{thm:stability-combined-ls}.
    Suppose \( K \) is even and that we teleport every-other iteration
    starting with \( k = 1 \).
    Then, GD with teleportation converges as
    \begin{equation}
        \delta_K
        \leq \frac{2 R^2 L (\tL - \tmu)}{\tmu \sbr{\rbr{\frac{\tL}{\tL - \tmu}}^{K/2} - 1}}.
    \end{equation}
    Moreover, this is strictly faster than GD without teleportation if
    \begin{equation}
        K >  \log\rbr{\frac{(\tL - \tmu)K + \tmu}{\tmu}} / \log\rbr{\frac{\tL}{\tL - \tmu}}.
    \end{equation}
\end{corollary}
\begin{proof}
    Using the fact that \( \calT = \cbr{1, 3, \ldots, K-1} \) implies
    \begin{align*}
        \delta_K
         & \leq \frac{2 R^2 L}{\sum_{i = 2; i \text{ even}}^{K}
        \rbr{\frac{\tL}{\tL - \tmu}}^{(K - i) / 2}}                                          \\
         & = \frac{2 R^2 L}{\sum_{j = 0}^{K/2 - 1}
        \rbr{\frac{\tL}{\tL - \tmu}}^{j}}                                                    \\
         & = \frac{2 R^2 L (\tL - \tmu)}{\tmu \sbr{\rbr{\frac{\tL}{\tL - \tmu}}^{K/2} - 1}}.
    \end{align*}
    This rate is faster than standard GD when
    \begin{align*}
        \frac{\tL - \tmu}{\tmu \sbr{\rbr{\frac{\tL}{\tL - \tmu}}^{K/2} - 1}}
         & < \frac{1}{K}                        \\
        \iff
        \frac{(\tL - \tmu)K + \tmu}{\tmu}
         & < \rbr{\frac{\tL}{\tL - \tmu}}^{K/2} \\
        \iff
        2 \log\rbr{\frac{(\tL - \tmu)K + \tmu}{\tmu}} / \log\rbr{\frac{\tL}{\tL - \tmu}}
         & < K.
    \end{align*}
\end{proof}

\begin{restatable}{theorem}{stabilityTeleport}\label{thm:stability-teleport}
    Suppose \( f \) is \( L \)-smooth, convex, and satisfies Hessian
    stability on \( \calS_0 \).
    Consider any teleportation schedule \( \calT \) and
    let \( M = K - \abs{\calT} \) be the
    number of steps without teleportation.
    Then GD with step-size \( \eta < \frac{2}{L \tL} \)
    converges as,
    \begin{equation}\label{eq:stability-teleport}
        \delta_K \leq
        \frac{2 R^2}
        {\xi M + 2 R^2 \sum_{k \in \calB} \frac{\psi_{k-1}}{\delta_{k-1}}},
    \end{equation}
    where \( \xi = \eta \rbr{2 - L \eta} \),
    \( \psi_{k-1} = \sbr{1 - \sbr{\prod_{i=k}^{k + b_k - 1} \beta_i}}\),
    and \( \beta_i = \rbr{1 + \lambda_i \eta \tmu \rbr{\tL \lambda_i \eta - 2}}^{-1} \).
\end{restatable}
\begin{proof}
    Since \( \eta < \frac{2}{L} \), \cref{lemma:diameter-bound} implies
    that the iterates --- with or without teleportation steps ---
    remain within the initial sub-level set.
    As a result, Hessian stability holds for all iterates.

    Assume \( k \in \calB \) and let
    \[
        \beta_k^{-1} = \rbr{1 + \lambda_k \eta \tmu \rbr{\tL \lambda_k \eta - 2}}.
    \]
    Since \( \lambda_k \leq L \) and \( \eta < \frac{2}{L \tL} \),
    we know that \( \beta_k \in \rbr{0, 1} \).
    Using \cref{lemma:stability-teleport-progress} and unrolling
    \cref{eq:stability-teleport-progress} for the \( b_k \) teleportation
    steps, we obtain
    \begin{align*}
        \delta_{k + b_k}
         & \leq
        \sbr{\prod_{i=k}^{k + b_k - 1} \beta_i^{-1}}
        \delta_{k}                                                                    \\
         & \leq \sbr{\prod_{i=k}^{k + b_k - 1} \beta_i^{-1}}
        \sbr{\delta_{k-1} - \eta\rbr{1 - \frac{L \eta}{2}} \norm{\grad(w_{k-1})}_2^2} \\
         & =  \sbr{\prod_{i=k}^{k + b_k - 1} \beta_i^{-1}}
        \sbr{\delta_{k-1} - \frac{\xi}{2} \norm{\grad(w_{k-1})}_2^2},
    \end{align*}
    where the second inequality follows from the descent lemma (\cref{eq:descent-lemma})
    and \( \xi = \eta \rbr{2 - L \eta} \).
    Re-arranging this inequality allows us to control the gradient norm
    as follows:
    \begin{equation}\label{eq:multi-step-descent}
        \begin{aligned}
            \norm{\grad(w_{k-1})}_2^2
             & \leq \frac{2}{\xi} \sbr{\delta_{k-1} - \delta_{k + b_k}} +
            \frac{2}{\xi} \sbr{1 - \sbr{\prod_{i=k}^{k + b_k - 1} \beta_i}}
            \delta_{k+b_k}
        \end{aligned}
    \end{equation}
    Starting now from the optimality gap and using convexity as well
    the Cauchy-Schwarz inequality, we obtain
    \begin{align*}
        \delta_{k-1}
         & \leq \abr{\grad(w_{k-1}), \w_{k-1} - \wopt} \\
         & \leq R \norm{\grad(\w_{k-1})}_2.
    \end{align*}
    To simplify the notation, it will be very helpful to define the scalar,
    \[
        \psi_{k-1} = \sbr{1 - \sbr{\prod_{i=k}^{k + b_k - 1} \beta_i}}.
    \]
    Substituting in \cref{eq:multi-step-descent} from above and then
    dividing on both sides by \( \delta_{k-1} \) and \( \delta_{k+b_k} \) gives
    the following:
    \begin{align}
        \delta_{k-1}^2
         & \leq \frac{2R^2}{\xi} \rbr{\delta_{k-1} - \delta_{k+b_k}}
        + \frac{2 R^2}{\xi} \sbr{1 - \sbr{\prod_{i=k}^{k + b_k - 1} \beta_i}}
        \delta_{k+b_k} \nonumber                                                  \\
        \implies
        1 \leq \frac{\delta_{k-1}}{\delta_{k+b_k}}
         & \leq \frac{2 R^2}{\xi} \rbr{\inv{\delta_{k+b_k}} - \inv{\delta_{k-1}}}
        + \frac{2 R^2}{\xi} \psi_{k-1} \inv{\delta_{k-1}} \nonumber               \\
        \implies
        \frac{\xi}{2 R^2}
         & \leq \inv{\delta_{k+b}} - \inv{\delta_{k-1}}
        + \psi_{k-1} \inv{\delta_{k-1}} \label{eq:stability-recursion}.
    \end{align}
    Since \( k \in \calB \), it must be that \( k + b_k \not \in \calB \).
    Then, either we have \( k + b_k + 1 \in \calB \) and
    \cref{eq:stability-recursion} holds with the choice of \( k' = k + b_k + 1 \),
    that is,
    \[
        \frac{\xi}{2 R^2}
        \leq \inv{\delta_{k + b_{k'} + b_k + 1}} - \inv{\delta_{k + b_k}}
        + \psi_{k + b_k} \inv{\delta_{k + b_k}},
    \]
    or \( k + b_k + 1 \) is also not a teleportation step
    and we obtain the following the simpler
    inequality (see \cref{prop:convex-teleport}):
    \[
        \frac{\xi}{2 R^2 L}
        \leq \inv{\delta_{k+b_k+1}} - \inv{\delta_{k+b_k}}.
    \]
    In either case, these two inequalities form a telescoping series
    with \cref{eq:stability-ls-recursion}.
    Note that this analysis only holds for \( k \in \calT \) such that
    \( k - 1 \) was a gradient step.
    In this special case where \( k = 0 \in \calT \),
    we must treat the update at \( k \) as a regular gradient step for
    \cref{eq:stability-ls-recursion} to hold.
    However, this causes no problems since every GD step after teleporting
    can also be analyzed as a vanilla GD step using the descent lemma.

    Summing over this telescoping series allows us to obtain
    the following:
    \begin{align*}
        \frac{\xi (K - |\calT|)}{2 R^2}
         & \leq \rbr{\inv{\delta_{K}} - \inv{\delta_0}}
        + \sum_{k \in \calB} \psi_{k-1} \inv{\delta_{k-1}}
        \\
         & \leq \inv{\delta_{K}}
        + \sum_{k \in \calB} \psi_{k-1} \inv{\delta_{k-1}}.
    \end{align*}
    Finally, re-arranging this expression yields
    \[
        \delta_K \leq
        \frac{2 R^2}
        {\xi (K - |\calT|) + 2 R^2 \sum_{k \in \calB}
            \psi_{k-1} \inv{\delta_{k-1}}}.
    \]
\end{proof}

\begin{corollary}
    In the setting of \cref{thm:stability-teleport}, if \( \eta = \frac{1}{L \tL} \),
    then
    \[
        \frac{1}{\xi} \leq L \tL, \quad \quad
        \beta_i \leq \frac{L \tL}{L \tL - \lambda_i \tmu}, \quad \quad
        \psi_{k-1} \leq \sbr{1 - \sbr{\prod_{i=k}^{k + b_k - 1} \frac{L \tL}{L \tL - \lambda_i \tmu}}},
    \]
    and GD with teleportation converges at the following rate:
    \begin{equation}\label{eq:stability-teleport-explicit}
        \delta_K \leq
        \frac{2 R^2 L \tL}
        {M + 2 R^2 L \tL \sum_{k \in \calB} \sbr{1 - \sbr{\prod_{i=k}^{k + b_k - 1} \frac{L \tL}{L \tL - \lambda_i \tmu}}} \frac{1}{\delta_{k-1}}}.
    \end{equation}
    Moreover, the dependence on \( L \tL \) can be improved to \( L \) if
    the step-size \( \eta = \frac{1}{L} \) is used whenever \( k \not \in \calT \).
    Finally, if teleportation steps are used every-other iteration, then
    GD with teleportation satisfies
    \begin{equation}\label{eq:stability-teleport-explicit-alternating}
        \delta_K \leq
        \frac{2 R^2 L \tL}
        {K/2 + 2 R^2 L \tL \sum_{k \in \calB} \sbr{\frac{\lambda_i \tmu}{L \tL - \lambda_i \tmu} \frac{1}{\delta_{k-1}}}}.
    \end{equation}
\end{corollary}

\begin{restatable}{theorem}{stabilityTeleportCombined}\label{thm:stability-teleport-combined}
    Suppose \( f \) is \( L \)-smooth, convex, and satisfies Hessian
    stability on \( \calS_0 \).
    Consider any teleportation schedule \( \calT \)
    and let \( M = K - \abs{\calT} \) be the
    number of steps without teleportation.
    Then GD with step-size \( \eta < \frac{2}{L \tL} \)
    converges as,
    \begin{equation}\label{eq:stability-teleport-combined}
        \delta_K
        \leq \frac{2 R^2}{\xi \sum_{k \not \in \calT, k > 0}
            \sbr{\prod_{i \not \in \calT, i > k} \zeta_{i-1}}},
    \end{equation}
    where \( \xi = \eta \rbr{2 - L \eta} \),
    \( \beta_i = \rbr{1 + \lambda_i \eta \tmu \rbr{\tL \lambda_i \eta - 2}}^{-1} \)
    and
    \[
        \zeta_{k-1} =
        \begin{cases}
            \sbr{\prod_{i=k}^{k + b_k - 1} \beta_i}
              & \mbox{if \( k \in \calB \)} \\
            1 & \mbox{otherwise.}
        \end{cases}
    \]
\end{restatable}
\begin{proof}
    Recall from the proof of \cref{thm:stability-teleport} that
    \( \xi = \eta \rbr{2 - L \eta} \) and
    \[
        \beta_k^{-1} = \rbr{1 + \lambda_k \eta \tmu \rbr{\tL \lambda_k \eta - 2}}.
    \]
    Using this notation, define the constant
    \[
        \zeta_{k-1} =
        \begin{cases}
            \sbr{\prod_{i=k}^{k + b_k - 1} \beta_i}
              & \mbox{if \( k \in \calB \)} \\
            1 & \mbox{otherwise.}
        \end{cases}
    \]
    Combing this with \cref{eq:stability-recursion} implies
    \begin{align*}
        \frac{\xi}{2 R^2}
         & \leq \inv{\delta_{k+b}} - \inv{\delta_{k-1}}
        + \psi_{k-1} \inv{\delta_{k-1}}                               \\
         & = \inv{\delta_{k+b_k}} - \frac{\zeta_{k-1}}{\delta_{k-1}}.
    \end{align*}
    Since \( k \in \calB \), it must be that \( k + b_k \not \in \calB \).
    If \( k + b_k + 1 \in \calB \), then
    \cref{eq:stability-ls-recursion} also holds with the choice of \( k' = k + b_k + 1 \)
    and
    \[
        \frac{\xi}{2 R^2 L}
        \leq \inv{\delta_{k + b_{k'} + b_k + 1}} - \frac{\zeta_{k+b_k}}{\delta_{k + b_k}}.
    \]
    On the other hand, if \( k + b_k + 1 \) is also not a teleportation step, then
    we have (see \cref{prop:convex-teleport}):
    \[
        \frac{\xi}{2 R^2}
        \leq \inv{\delta_{k+b_k+1}} - \frac{\zeta_{k+b_k}}{\delta_{k+b_k}}.
    \]
    Let \( \theta_0 = 1 \) and \( \theta_{k+b_k} = \theta_{k-1} / \zeta_{k-1} \).
    We must choose the next \( \theta \) to do a weighted telescoping of these equations.
    If \( k + b_k + 1 \in \calT \), then set \( \theta_{k+b_{k'}+b_k+1} = \theta_{k + b_k} / \zeta_{k+b_k}  \).
    On the other hand, if \( k + b_k + 1 \not \in \calB \),
    the we must set \( \theta_{k+b_k+1} = \theta_{k+b_k} / \zeta_{k+b_k} \).
    In either case, weighting the telescoping series with the \( \theta_i \)'s
    and summing gives the following:
    \begin{align*}
        \frac{\xi \theta_{k+b_k}}{2 R^2}
         & \leq \frac{\theta_{k+b_k}}{\delta_{k+b_k}} - \frac{\theta_{k-1}}{\delta_{k-1}} \\
        \implies
        \frac{\sum_{k \not \in \calT} \theta_{k}}{2 R^2 L}
         & \leq \frac{\theta_{K}}{\delta_{K}} - \frac{\theta_{0}}{\delta_{0}}             \\
         & \leq \frac{\theta_{K}}{\delta_{K}}.
    \end{align*}
    Re-arranging this equation gives
    \begin{equation}\label{eq:series-dependence}
        \delta_K
        \leq \frac{2 R^2}{\xi \sum_{k \not \in \calT, k > 0} \theta_{k} / \theta_{K}}.
    \end{equation}
    To determine the rate of convergence, we must compute
    \( \theta_{k} / \theta_{K} \) for each \( k \not \in \calT \).

    By construction, the ratio satisfies
    \begin{align*}
        \theta_{k}
        = \prod_{i \not \in \calT, i \leq k} \zeta_{i-1}
        \sbr{\prod_{i \not \in \calT, i \leq k} \zeta_{i-1}}^{-1}
        = \prod_{i \not \in \calT, i > k} \zeta_{i-1},
    \end{align*}
    meaning our final convergence rate is given by,
    \[
        \delta_K
        \leq \frac{2 R^2}{\xi \sum_{k \not \in \calT, k > 0}
            \sbr{\prod_{i \not \in \calT, i > k} \zeta_{i-1}}}.
    \]
\end{proof}

\begin{proposition}
    In the setting of \cref{thm:stability-teleport}, if \( \eta = \frac{1}{L \tL} \),
    teleportation steps are used every-other iteration, and \( K \geq 4 \tL / \tmu \),
    then GD with teleportation satisfies
    \begin{equation}\label{eq:stability-teleport-explicit-alternating-combined}
        \delta_K
        \leq \frac{4 R^2 L \tL^2}{\tmu} \sbr{\prod_{i=2; i \text{ even}}^{K} \rbr{1 - \frac{\lambda_k \tmu}{L \tL}}}.
    \end{equation}
    Moreover, this rate can be tightened by a factor of \( \tL \) if the step-size
    \( \eta = \frac{1}{L} \) is used at iterations \( k \not \in \calT \).
\end{proposition}
\begin{proof}
    Our choice of step-size implies
    \[
        \frac{1}{\xi} \leq L \tL, \quad \quad
        \beta_i^{-1} = 1 + \frac{\lambda_k \tmu}{L \tL}\rbr{\frac{\lambda_k}{L} - 2}, \quad \quad
        \zeta_{k-1} = \beta_i^{-1},
    \]
    Similarly, bounding
    \( 0 \leq \lambda_k \leq L \) and \( \eta = \frac{1}{L \tL}  \) gives
    \begin{align} \label{eq:tempzoe8hd9zd}
        1 - \frac{2\tmu}{\tL} \leq 1 - \frac{2\lambda_k \tmu}{L \tL}\leq \;\;
        \zeta_{k-1}^{-1}
         & = 1 + \frac{\lambda_k \tmu}{\tL L} \rbr{\frac{\lambda_k}{L} - 2}
        \;\;\leq 1 - \frac{\lambda_k \tmu}{L \tL}.
    \end{align}
    Using the geometric series we have that
    \begin{align*}
        \sum_{k=2; k \text{ even}}^{K} \theta_k
         & =
        \sum_{k=2; k \text{ even}}^{K} \sbr{\prod_{i=k+2; i \text{ even}}^{K}}
        \gamma_i^{-1}                                                        \\
         & \geq
        \sum_{j=0}^{K/2 - 1}  \rbr{1 - \frac{2\tmu}{\tL}}^j                  \\
         & = \frac{1 - \left(1 - \frac{2\tmu}{\tL}\right)^{K/2}}{2\tmu/\tL}.
    \end{align*}
    Let \( \kappa = \frac{\tL}{\tmu} \)
    and choose $K$ so that
    \begin{equation}\label{eq:zzerzerezr}
        \left( 1 - \frac{2}{\kappa^{-1}}\right)^{K/2}  \leq \frac{1}{2}
        \quad \iff \quad  \frac{\log(2)}{\log\rbr{1 + \frac{2}{\kappa^{-1} - 2}}} \leq  K/2.
    \end{equation}
    To simplify the above expression we can use,
    \[
        \frac{x}{1+x}\leq \log(1+x) \leq x, \quad \mbox{ for }x > -1,
    \]
    since \( \kappa^{-1} \leq 1 \), with strictness or else one step of GD
    after teleporting solves the optimization problem.
    Using this inequality implies
    \[
        \frac{\log(2)}{\log\rbr{1 + \frac{2}{\kappa^{-1} - 2}}}
        \leq
        2\kappa \leq \frac{K}{2} \quad \iff \quad  \frac{4\tL}{\tmu} \geq K.
    \]
    Using the above to control the geometric series gives,
    \begin{align} \label{eq:tempzoe8hd9zd3}
        \sum_{k=2; k \text{ even}}^{K} \sbr{\prod_{i=k+2; i \text{ even}}^{K}}
        \gamma_i^{-1}
        \geq
        \frac{\tL}{4\tmu},
    \end{align}
    if \( K \geq 4\tL/\tmu \).
    Using this together with \cref{eq:series-dependence}
    gives that the convergence rate is at least
    \begin{align*}
        \delta_K
         & \leq \frac{2 R^2}{\xi \sum_{k \not \in \calT, k > 0} \theta_{k} / \theta_{K}}                                 \\
         & \leq \frac{4 R^2 L \tL^2}{\tmu} \sbr{\prod_{i=2; i \text{ even}}^{K} \rbr{1 - \frac{\lambda_k \tmu}{L \tL}}},
    \end{align*}
    if \( K \geq 4 \tL / \tmu \).

    To see that the rate can be tightened, observe that using \( \eta_{k-1} = 1/L \)
    when \( k-1 \not \in \calT \) tightens the value of \( \xi \) to
    \( \frac{1}{\xi} = L \) without affecting the proof.
\end{proof}

\section{EVALUATING THE TELEPORTATION OPERATOR: PROOFS}\label{app:algorithm}


\updateRule*
\begin{proof}
    We proceed by case analysis.
    Let \[ \bar \x = \argmax_{x \in \R^d} \cbr{\half \log(\norm{\grad(\xk)}_2^2) +
            \abr{\frac{\nabla^2 f(\xk) \nabla f(\xk)}{\norm{\grad(\xk)}_2^2}, x - \xk} -
            \frac{1}{2\rho_t} \norm{\x - \xk}_2^2}.
    \]
    Taking derivatives shows that \( \bar x \) is a gradient ascent step with
    step-size \( \rho_t \),
    \[
        \bar \x
        = \xk + \rho_t \frac{\nabla^2 f(\xk) \nabla f(\xk)}{\norm{\grad(\xk)}_2^2}.
    \]

    \textbf{Case 1}: \( \bar \x \in \tilde \calS_k(\xk) \).
    Then \( \bar x \) satisfies the linearized constraint and \( \xkk = \bar x \)
    must hold.
    We conclude that,
    \[
        \xkk
        = \xk + \rho_t \frac{\nabla^2 f(\xk) \grad(\xk)}{\norm{\grad(\xk)}_2^2}.
    \]
    Substituting this value into the linearization of the half-space
    and using \( \xkk \in \tilde \calS_k(\xk) \), we
    find
    \[
        \rho_t \frac{\grad(\xk)^\top \nabla^2 f(\xk) \grad(\xk)}{\norm{\grad(\xk)}_2^2}
        + f(\xk) - f(\wk) \leq 0.
    \]

    \textbf{Case 2}: \( \bar \x \not \in \tilde \calS_k(\xk) \).
    Then the solution lies on the boundary of the half-space constraint and is
    given by projecting \( \bar \x \) onto \[ \tilde \calL_k(\xk) = \cbr{x :
            f(\xk) + \abr{\grad(\xk), x - \xk} = f(\wk)}.
    \]
    The Lagrangian of this problem is
    \[
        L(x, \lambda) = \half \norm{\x - \bar \x}_2^2
        + \lambda \rbr{f(\xk) + \abr{\grad(\xk), x - \xk} - f(\wk)}.
    \]
    Minimizing over \( x \) yields
    \[
        \xkk = \bar \x - \lambda \grad(\xk),
    \]
    which shows that the dual function is given by,
    \[
        \begin{aligned}
            d(\lambda)
             & = - \half \lambda^2 \norm{\grad(\xk)}_2^2
            + \lambda \rbr{f(\xk) + \abr{\grad(\xk), \bar x - \xk} - f(\wk)}                                                         \\
             & = - \half \lambda^2 \norm{\grad(\xk)}_2^2
            + \lambda \rbr{f(\xk) + \frac{\rho_t}{\norm{\grad(\xk)}^2_2} \abr{\grad(\xk), \nabla^2 f(\xk) \nabla f(\xk)^2} - f(\wk)} \\
        \end{aligned}
    \]
    This is a concave quadratic; maximizing over \( \lambda \) gives the
    following dual solution:
    \[
        \lambda^* = \frac{f(\xk) - f(\wk)
            + \rho_t \abr{\grad(\xk), \nabla^2 f(\xk) \nabla f(\xk)} / \norm{\grad(\xk)}_2^2}{\norm{\grad(\xk)}_2^2}.
    \]
    Note that
    \[
        f(\xk) - f(\wk)
        + \rho_t \abr{\grad(\xk), \nabla^2 f(\xk) \nabla f(\xk)} / \norm{\grad(\xk)}_2^2
        > 0,
    \]
    since the half-space constraint is violated at \( \bar x \).
    Plugging this value back into the expression for \( \xkk \),
    \begin{align*}
        \xkk & =  \xk + \rho_t \frac{\nabla^2 f(\xk) \grad(\xk)}{\norm{\grad(\xk)}_2^2} \\
             & \hspace{4em}
        - \frac{\rho_t \abr{\grad(\xk), \nabla^2 f(\xk) \grad(\xk)} / \norm{\grad(\xk)}_2^2
            + f(\xk) - f(\wk)}{\norm{\grad(\xk)}_2^2} \grad(\xk).
    \end{align*}
    This completes the second case.
    Putting the analysis together, we obtain the desired result:
    \begin{align*}
        \xkk & =  \xk + \rho_t \frac{\nabla^2 f(\xk) \grad(\xk)}{\norm{\grad(\xk)}_2^2} \\
             & \hspace{4em}
        - \rbr{\frac{\rho_t \abr{\grad(\xk), \nabla^2 f(\xk) \grad(\xk)} / \norm{\grad(\xk)}_2^2
                + f(\xk) - f(\wk)}{\norm{\grad(\xk)}_2^2}}_+ \grad(\xk).
    \end{align*}
\end{proof}

\sqpConvergence*
\begin{proof}
    The convergence guarantees follow from the connection between
    \cref{eq:teleportation-sub-problem} and SQP as developed by
    \citet{torrisi2018projected}.
    \cref{eq:update-rule} gives an exact solution to the SQP
    update in \citet[Equation 2.3]{torrisi2018projected} with
    the setting \( H^{(i)} = \frac{1}{\rho_t} \) and the
    parameterization \( d_z = \xkk - \xk \).
    To obtain their guarantee, we must also explicitly form
    the dual variables \( \lambda_t \) associated with the linearized
    constraint, but these are only
    needed to select the step-size \( \alpha_t \) used in the relaxation
    step.
    Note that the values for the dual variables can be computed in closed form
    from the KKT conditions for the linearized constraint.

    Now, we do require \( (\alpha_t, \rho_t) \) to be ``appropriately
    selected'' for convergence guarantees to hold.
    \citet{torrisi2018projected} choose \( \alpha_t \leq 1 \) using the strong
    Wolfe conditions \citep{wolfe1969convergence, wolfe1971convergence} on
    the augmented Lagrangian function (this explains the need to maintain
    explicit values for the dual parameter \( \lambda_t \)).
    Note that the augmented Lagrangian function also requires a penalty
    parameter \( \delta_t \).
    \citet{torrisi2018projected} set this parameter to ensure descent
    in a similar fashion to our choice for the merit function
    in \cref{prop:penalty-strength}.

    No conditions on \( \rho_t > 0 \) are required for global convergence of the
    algorithm except boundedness.
    For local linear convergence, the step-size \( \rho_t \) must satisfy
    \( \rho_t \leq 1 / \lambda_{\text{max}}\rbr{\nabla^2 \calL(x^*, \lambda^*)} \).
    This is a typical assumption for fast local convergence of SQP methods.
    The requirement that \( \grad(\wk) \neq 0 \) is sufficient to guarantee
    that implies LICQ holds, as discussed in \cref{sec:teleport}.
\end{proof}

\begin{lemma}
    \label{lemma:directional-derivative}
    The directional derivative of the merit function satisfies
    \begin{equation}
        \label{eq:directional-derivative}
        D_{\phi}(\xk; d_t)
        \geq \rbr{\rho_t \norm{\nabla^2 f(\xk) \nabla f(x)}_2^2 - \abr{\nabla^2 f(\xk) \nabla
                f(\xk), v_t} } /
        \norm{\grad(\xk)}_2^4 + \gamma \rbr{f(\xk) - f(\wk)}_+.
    \end{equation}
    Moreover, if
    \( \gamma_t > \abr{q_t, v_t} / \norm{\grad(\xk)}_2^4 (f(x_t) - f(w_k))_+ \),
    then \( d_t \) is a ascent direction for
    \( \phi_{\gamma} \).
\end{lemma}
\begin{proof}
    Let \( d_t = \xkk - \xk \).
    Define \( \Delta_t(\alpha) = \phi_\gamma(w^{t} + \alpha d_t) -
    \phi_\gamma(\xk) \).
    Using first-order Taylor expansions, we have
    \begin{align*}
        \Delta_t(\alpha)
         & = \half \log(\norm{\nabla f(\xk + \alpha d_t)}_2^2)
        - \gamma \rbr{f(\xk + \alpha d_t) - f(\wk)}_+
        -  \half \log(\norm{\nabla f(\xk)}_2^2) + \gamma \rbr{f(\xk) - f(\wk)}_+          \\
         & = \frac{\alpha}{\norm{\grad(\xk)}_2^2}\abr{\nabla^2 f(\xk) \nabla f(\xk), d_t}
        - \gamma \rbr{f(\xk) + \alpha \abr{\nabla f(\xk), d_t} - f(\wk)}_+
        + \gamma \rbr{f(\xk) - f(\wk)}_+                                                  \\
         & \hspace{5em} + O(\alpha^2)                                                     \\
         & \geq \frac{\alpha}{\norm{\grad(\xk)}_2^2}
        \abr{\nabla^2 f(\xk) \nabla f(\xk), d_t}
        - \gamma (1-\alpha) \rbr{f(\xk) - f(\wk)}_+
        + \gamma \rbr{f(\xk) - f(\wk)}_+
        + O(\alpha^2),                                                                    \\
        \intertext{
            where we have used \( \abr{\nabla f(\xk), d_t} \leq f(\wk) - f(\xk) \)
            since \( \xkk \) satisfies the linearized sub-level set constraint.
            Simplifying, we obtain, }
        \Delta_t(\alpha)
         & \geq \frac{\alpha}{\norm{\grad(\xk)}_2^2} \abr{\nabla^2
            f(\xk) \nabla f(\xk), d_t} + \gamma \alpha \rbr{f(\xk) - f(\wk)}_+ +
        O(\alpha^2).
    \end{align*}
    Dividing both sides by \( \alpha \) and taking the limit as \( \alpha \rightarrow 0 \)
    shows that
    \begin{align*}
        D_{\phi}(\xk; d_t)
         & \geq \frac{1}{\norm{\grad(\xk)}_2^2}\abr{\nabla^2 f(\xk) \nabla f(\xk), d_t} +
        \gamma \rbr{f(\xk) - f(\wk)}_+                                                        \\
         & = \rbr{\rho_t \norm{\nabla^2 f(\xk) \nabla f(x)}_2^2 - \abr{\nabla^2 f(\xk) \nabla
                f(\xk), v_t} } /
        \norm{\grad(\xk)}_2^4 + \gamma \rbr{f(\xk) - f(\wk)}_+.
    \end{align*}
    Substituting the choice of \( \gamma_t \) into the directional derivative yields,
    \begin{align*}
        D_{\phi}(\xk; d_t)
         & \geq \rbr{\rho_t \norm{\nabla^2 f(\xk) \nabla f(x)}_2^2 -
            \abr{\nabla^2 f(\xk) \nabla f(\xk), v_t} } / \norm{\grad(\xk)}_2^4
        + \gamma \rbr{f(\xk) - f(\wk)}_+                                           \\
         & > \rho_t \norm{\nabla^2 f(\xk) \nabla f(x)}_2^2 / \norm{\grad(\xk)}_2^4
        \geq 0.
    \end{align*}
    Since this quantity is strictly positive, \( d_t \) is an ascent direction
    for \( \phi_\gamma \).
\end{proof}

\penaltyStrength*
\begin{proof}
    The first part of the proof follows immediately from
    \cref{lemma:directional-derivative}.
    Substituting the choice of \( \gamma_t \) into the line search condition yields,
    \begin{align*}
        \phi_\gamma(\xkk)
         & \geq \phi_\gamma(\xk) + D_\phi(\xk, \xkk - \xk) / 2                   \\
         & > \phi_\gamma(\xk) + \rho_t \norm{\nabla^2 f(\xk) \nabla f(x)}_2^2
        / (2 \norm{\grad(\xk)}_2^4)
    \end{align*}
    which is straightforward to check in practice.
\end{proof}

\begin{proposition}
    \label{prop:relu-teleport}
    Suppose \( g \) is a loss function and \( f(\w) = g(h_w(X), y), \) where \[
        h_w(X) = \phi(W_l \phi(\ldots W_2 (\phi(W_1 X))) \] is the prediction function
    of a neural network with weights \( w = \rbr{W_l, \ldots, W_1} \), \( l \geq 2
    \).
    If the activation function \( \phi \) is positively homogeneous such that \(
    \lim_{\beta \rightarrow \infty} \phi(\beta) = \infty \) and \( \lim_{\beta
        \rightarrow \infty} \phi'(\beta) < \infty \), then optimal value of the
    sub-level set teleportation problem is unbounded and
    \cref{eq:sublevel-teleport} does not admit a finite solution.
\end{proposition}
\begin{proof}
    For simplicity, we prove the result in the scalar case for \( l = 2 \),
    although it immediately generalizes.
    Since \( \phi \) is positively homogeneous, we have \[ w_2 \phi(w_1 x) =
        \alpha w_2 \phi((w_1 / \alpha) x).
    \]
    Let \( \tilde w_2 = \alpha w_2 \) and \( \tilde w_1 = w_1 / \alpha \).
    Define \( v = w_2 \phi(w_1 x) = \tilde w_2 \phi(\tilde w_1 x)\).
    Then gradients with respect to the first and second layers are given by
    \begin{align*}
        \frac{\partial}{\partial \tilde w_1} f(\w)
         & = \frac{\partial}{\partial v} g(v) w_2 \phi'(\tilde w_1 x) x          \\
         & = \alpha \frac{\partial}{\partial v} g(v) w_2 \phi'(w_1 x / \alpha) x \\
        \frac{\partial}{\partial \tilde{w_2}} f(\w)
         & = \frac{\partial}{\partial v} g(v) \phi(\tilde w_1 x)                 \\
         & = \frac{\partial}{\partial v} g(v) \phi(w_1 x / \alpha).
    \end{align*}
    Taking the limit as \( \alpha \rightarrow 0 \), we see that
    \begin{align*}
        \frac{\partial}{\partial \tilde w_1} f(\w) & \rightarrow 0       \\
        \frac{\partial}{\partial \tilde w_2} f(\w) & \rightarrow \infty,
    \end{align*}
    by assumption on \( \phi \).
    Thus, there exists a diverging sequence of points on the level set whose
    gradient norm is also diverging.
    Since the level sets are unbounded, the objective is not coercive and the
    problem is ill-posed.
    This completes the proof.
\end{proof}

\section{EXPERIMENTS}\label{app:experiments}


\subsection{Additional Experiments}\label{app:additional-experiments}

Now we provide further experimental results which could not be included in the
main paper due to space constraints.

\begin{figure*}[t]
    \centering
    \includegraphics[width=0.98\textwidth]{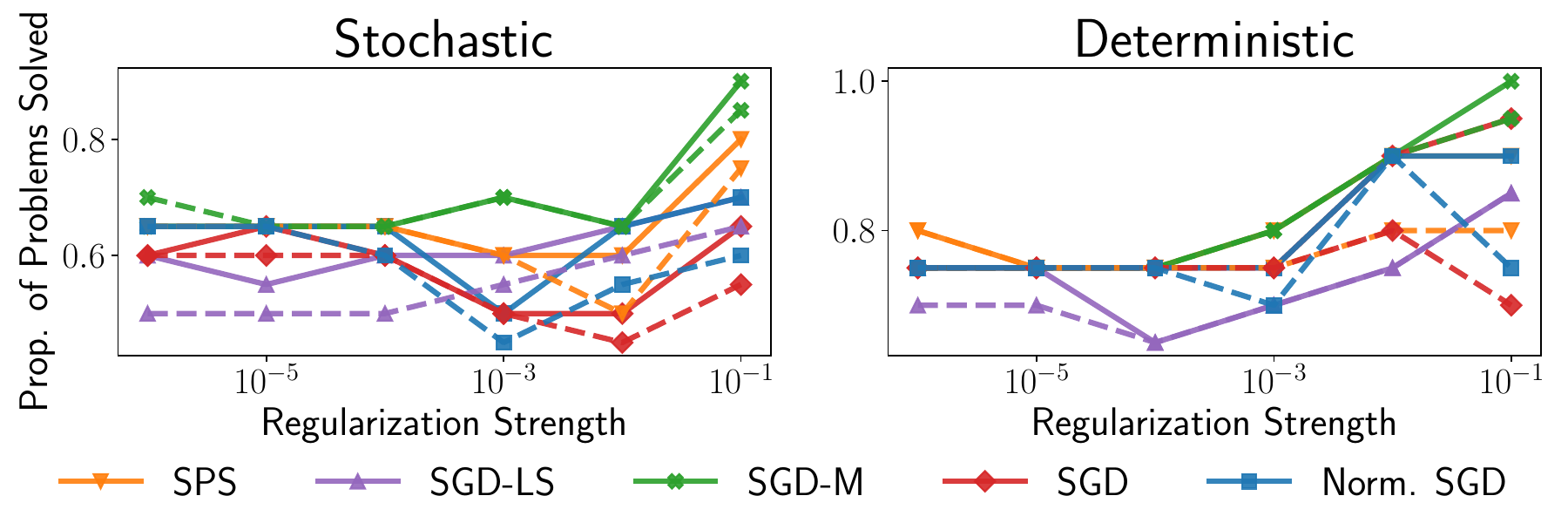}
    \caption{Effect of regularization strength on optimization when training
        three-layer ReLU networks on 20 datasets from the UCI repository.
        Success ratios are computed using the same methodology as in
        \cref{fig:network-profile}.
        Methods using teleportation steps (solid lines) outperform methods
        without (dashed) when regularization is moderate.
    }%
    \label{fig:regularization-ablation}
\end{figure*}

\textbf{Effects of Regularization}:
The teleportation problem may not admit a solution when
training neural networks with homogeneous activations (e.g. ReLU) since
the objective is not coercive and the level sets are not compact
(see \cref{prop:relu-teleport}).
In practice, we ``compactify'' the level sets using regularization.
\cref{fig:regularization-ablation} compares the strength of weight decay
regularization against the success rates of methods with and without
teleportation.
As expected, all methods solve more problems when the regularization is large.
However, while our intuition suggests that methods with teleportation should solve
more problems than standard gradient methods when the regularization is large,
this is not obviously true.
For example, standard SGD-M outperforms SGD-M with teleportation when regularization
is both very large and very small.
This indicates that teleportation may work best with moderate regularization.

\begin{figure}[t]
    \centering
    \includegraphics[width=0.98\textwidth]{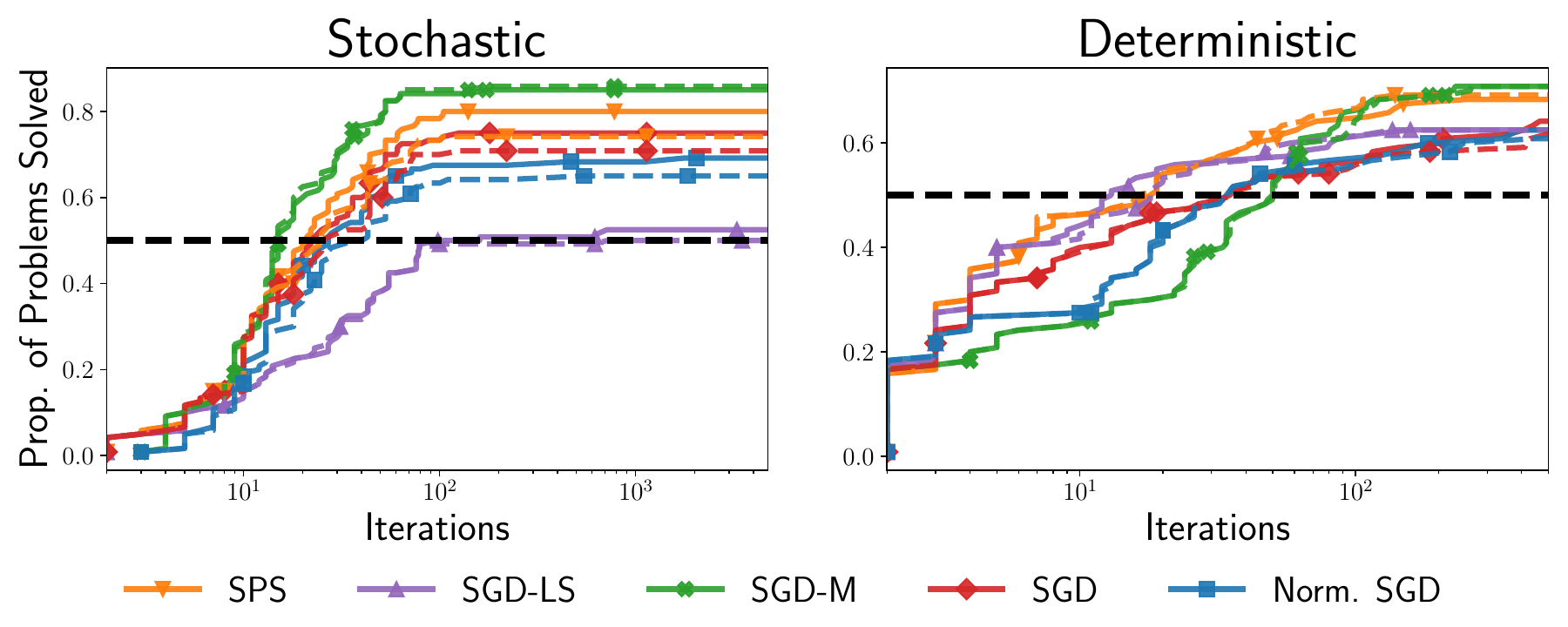}
    \caption{
        Performance profile comparing stochastic and deterministic
        optimization methods with (solid lines) and without (dashed lines)
        sub-level set teleportation for logistic regression models.
        Stochastic methods teleport once every \( 10 \) epochs starting from the fifth
        epoch, while deterministic methods teleport once every \( 50 \) iterations
        starting from iteration five.
        A problem is solved when \( \rbr{f(\wk) - f(\wopt)} / (f(w_0) - f(\wopt)) \leq
        \tau \), where \( f(\wopt) \) is estimated separately and \( \tau \) is a threshold.
    }%
    \label{fig:logistic-profile}
\end{figure}

\textbf{Logistic Regression Performance Profile}:
We generate an alternative version of the performance profile from the
main paper (\cref{fig:network-profile}) using the logistic regression model
class instead of two-layer neural networks.
Logistic regression is non-strongly convex, but satisfies the Hessian
stability condition we use in \cref{sec:stability} to establish fast
convergence rates for GD with teleportation.
Thus, we expect gradient methods with teleportation to perform better than
their counterparts without.

\cref{fig:logistic-profile} shows that including intermittent teleportation
steps improves the performance of every method.
This is particularly true for SGD with momentum (SGD-M), which solves nearly
every problem when augmented with teleportation steps in both the stochastic
and deterministic settings.
The poor performance of SGD with line-search (SGD-LS) in the stochastic
setting is due to the unreliability of step-sizes chosen by searching along
stochastic gradients, which may not be descent directions with respect to the
true gradient.
Typically interpolation is required for SGD-LS to converge when used with
stochastic gradients \citep{vaswani2019painless}.
Interpolation does not hold in our setting since we consider regularized
problems.

\begin{figure}[t]
    \centering
    \includegraphics[width=0.98\textwidth]{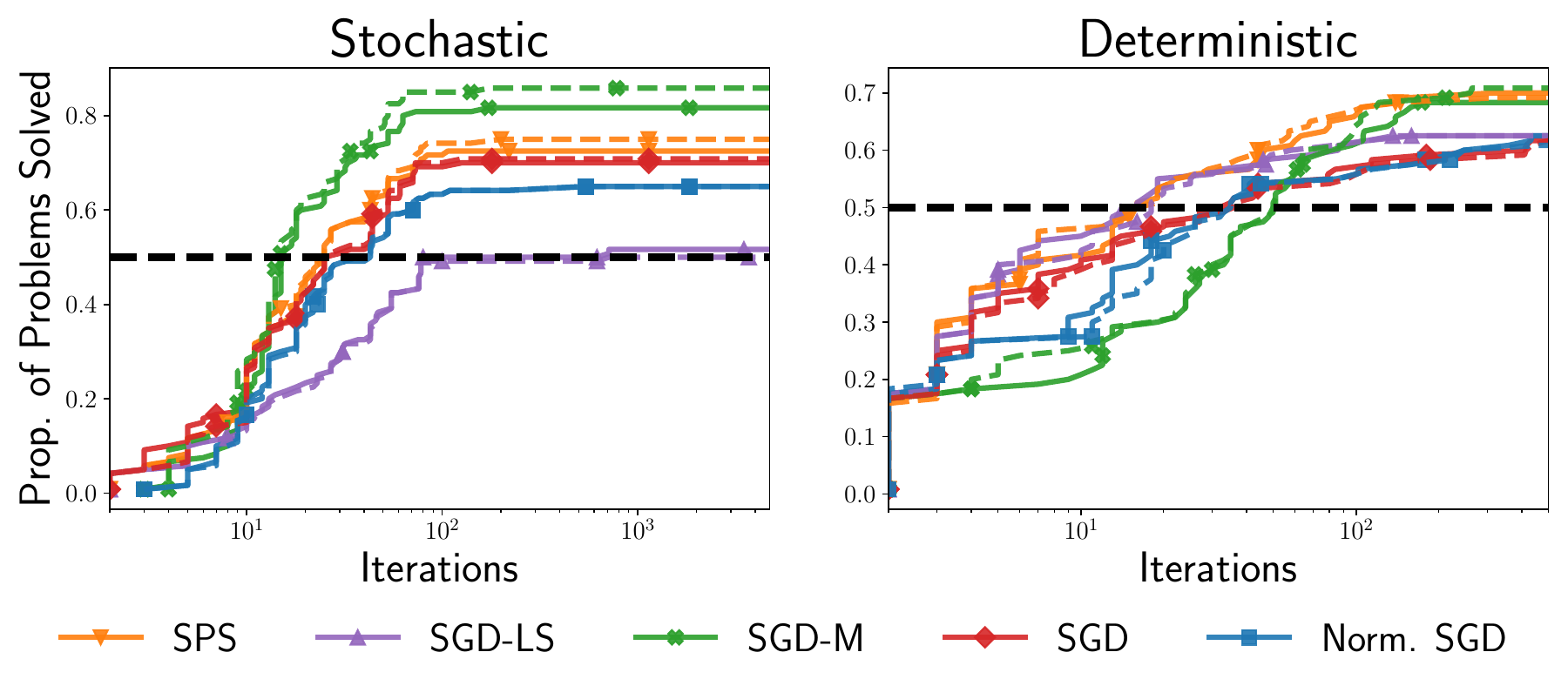}
    \caption{
        Performance profile comparing stochastic and deterministic
        optimization methods with (solid lines) and without (dashed lines)
        initialization by teleportation for training logistic regression models.
        The experimental configuration is the same as in \cref{fig:logistic-profile}.
        Initialization by teleportation harms the performance
        of all methods.
    }%
    \label{fig:logistic-profile-init}
\end{figure}

\begin{figure}[t]
    \centering
    \includegraphics[width=0.98\textwidth]{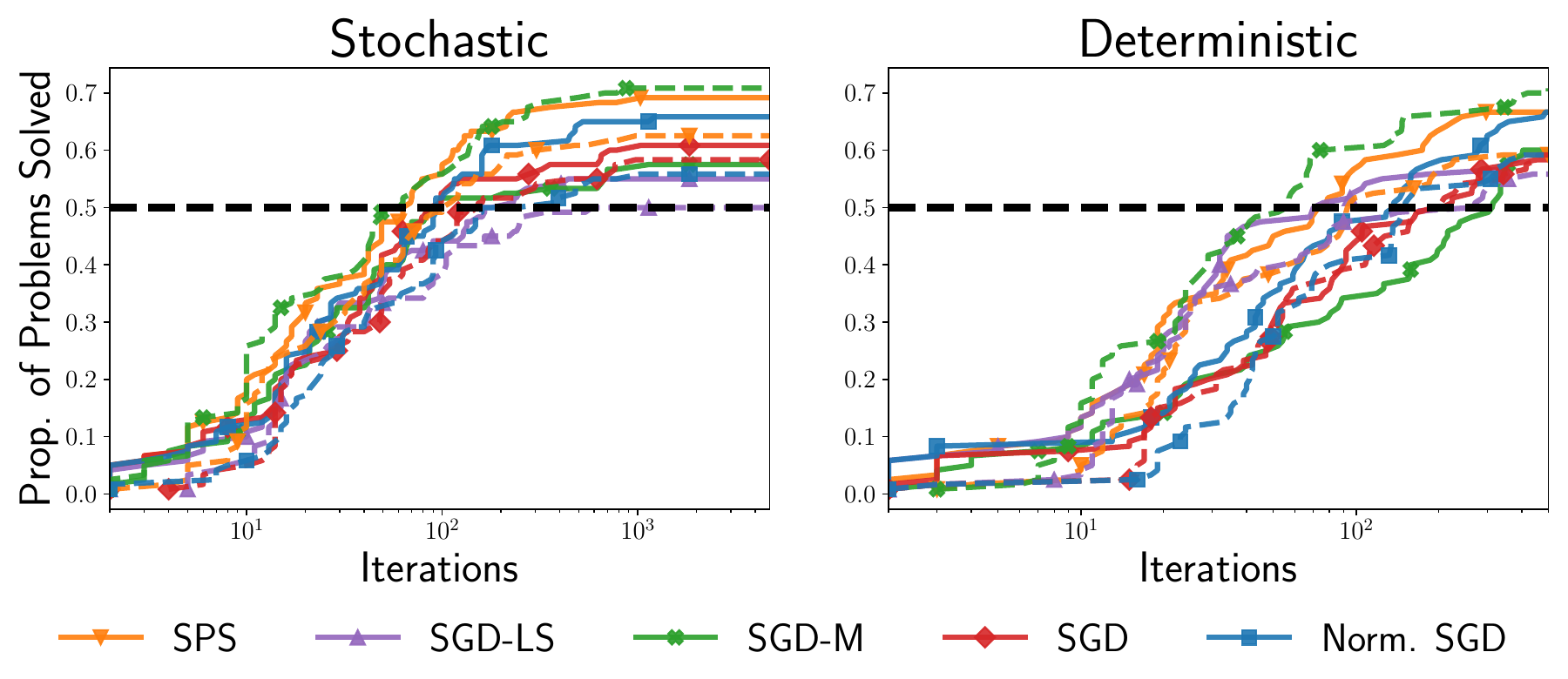}
    \caption{
        Performance profile comparing stochastic and deterministic
        optimization methods with (solid lines) and without (dashed lines)
        initialization by teleportation for training three-layer ReLU networks.
        The experimental configuration is the same as in \cref{fig:logistic-profile}.
        Initialization by teleportation slightly improves the performance of
        methods with adaptive step-sizes (SPS, Norm. SGD), but greatly worsens
        that of SGD-M.
    }%
    \label{fig:network-profile-init}
\end{figure}

\textbf{Initialization by Teleportation}:
One notable aspect of our theory is that it suggests not teleporting at the first
iteration.
This is because our proof technique requires that each block of iterations
with teleportation be unrolled back to a single GD step without teleportation;
this GD step then allows us to control the norm of the gradient using the
descent lemma (\cref{eq:descent-lemma}).
We overcome this limitation by simply treating the GD step at \( w_0 \) as
a standard gradient update regardless of teleportation, but the result is that
we do not benefit theoretically from teleporting at \( w_0 \).
Thus, it is interesting to examine whether very early teleportation steps are
useful for optimization.

Figures~\ref{fig:logistic-profile-init} and~\ref{fig:network-profile-init}
present performance profiles for gradient methods with teleportation at \( w_0
\) only (i.e. initialization by teleportation) and methods without
teleportation.
\cref{fig:logistic-profile-init} shows results for logistic regression,
while \cref{fig:network-profile-init} concerns training three-layer ReLU networks.
Aside from teleportation schedule, the experimental setup is identical to that
for Figures~\ref{fig:logistic-profile} and~\ref{fig:network-profile}.

We find that initialization by teleportation is broadly detrimental for
logistic regression problems, but useful for optimization methods which use
adaptive step-sizes (SPS, Norm. SGD, SGD-LS) when training ReLU networks.
These conclusions are particularly true for training neural networks
in the stochastic setting.
We conjecture this behavior is because of the strong curvature of
optimization problems near the initialization, which is exacerbated by
teleportation.
Optimizers with adaptive step-sizes can leverage this curvature, while SGD-M
and SGD cannot.

\begin{figure}
    \centering
    \includegraphics[width=0.81\textwidth]{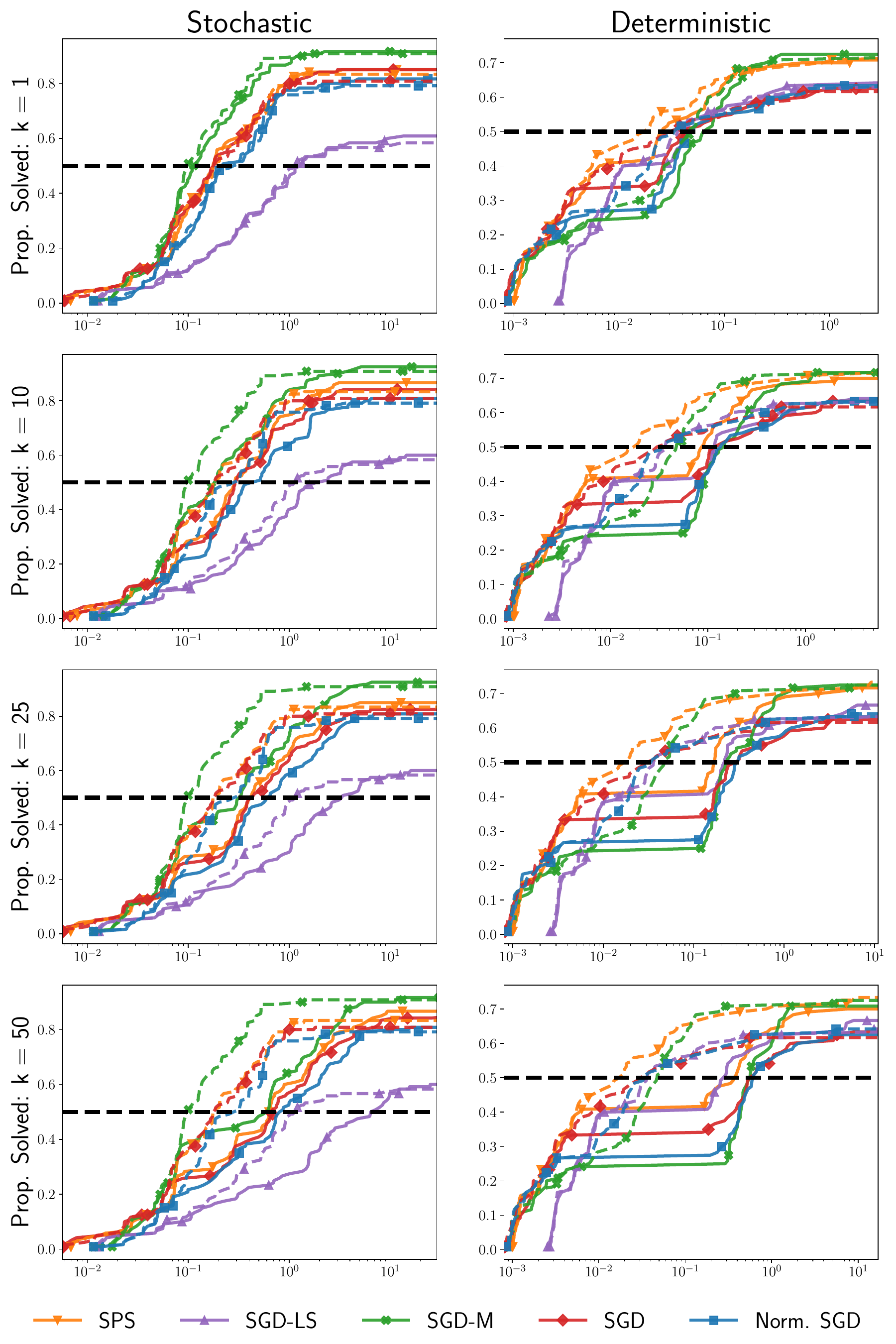}
    \caption{
        Performance profile comparing stochastic and deterministic
        optimization methods with (solid lines) and without (dashed lines)
        teleportation on \( 120 \) logistic regression problems.
        We restrict teleportation methods to \( k \) steps of
        \cref{alg:teleport} to control the time cost and accuracy of
        the teleportation sub-solver.
        Unlike \cref{fig:logistic-profile}, the x-axis shows cumulative time,
        including the time required to teleport, against the proportion of problems
        solved.
        Although teleportation introduces a significant computational overhead,
        methods with teleportation eventually catch-up to those without.
    }%
    \label{fig:logistic-profile-time}
\end{figure}

\begin{figure}
    \centering
    \includegraphics[width=0.81\textwidth]{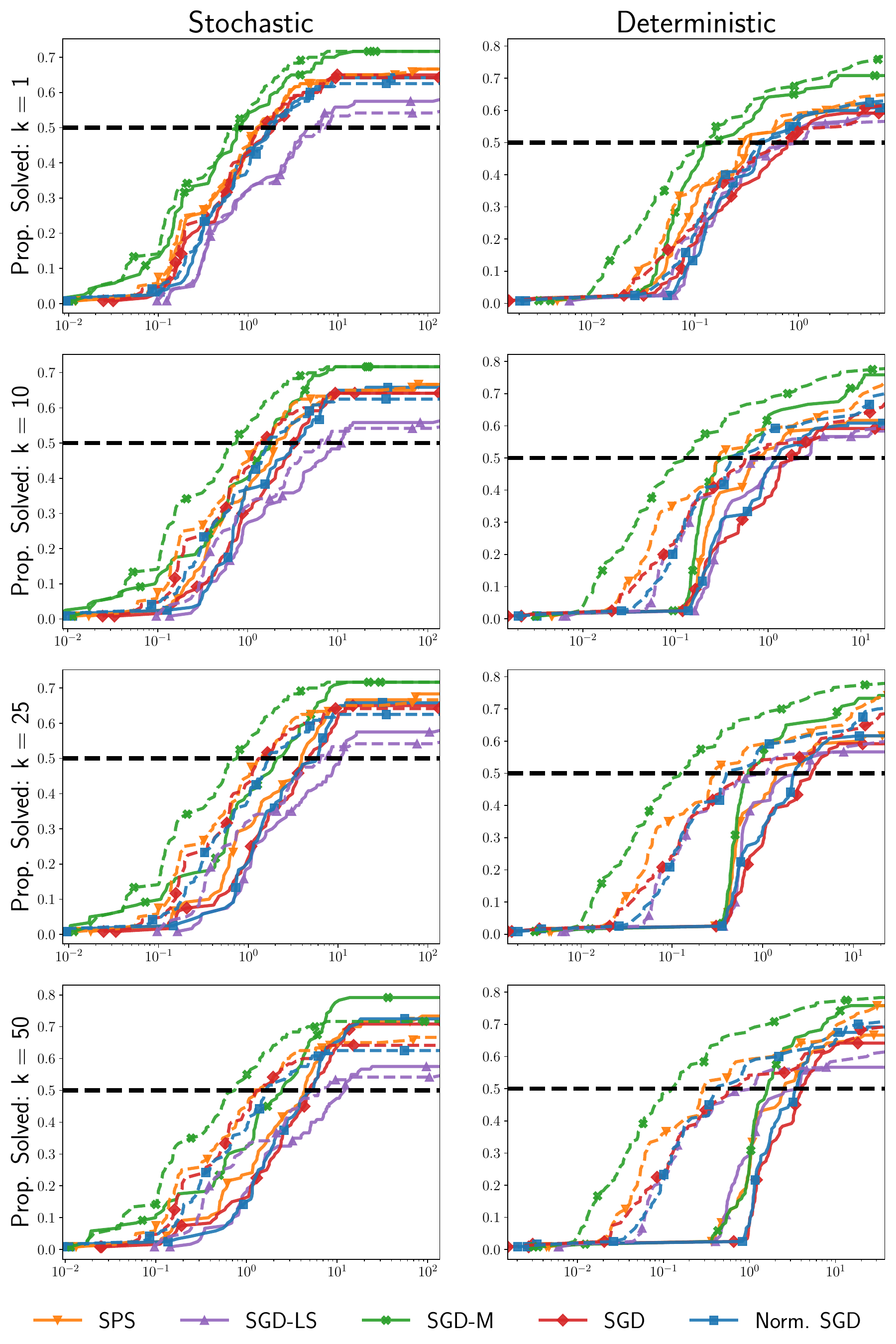}
    \caption{
        Performance profile comparing stochastic and deterministic
        optimization methods with (solid lines) and without (dashed lines)
        teleportation for training three-layer ReLU networks.
        We restrict teleportation methods to \( k \) steps of
        \cref{alg:teleport} to control the time cost and accuracy of
        the teleportation sub-solver.
        The x-axis shows cumulative time as in \cref{fig:logistic-profile-time}.
        We find solving the teleportation sub-problem to high accuracy is critical to
        good performance of gradient methods with teleportation.
    }%
    \label{fig:network-profile-time}
\end{figure}

\textbf{Timed Performance Profiles}:
In practice, evaluating the teleportation oracle introduces a computational
overhead.
The amount of overhead and the degree of accuracy to which the
teleportation sub-problem is solved is determined the number of iterations \( k \)
used by \cref{alg:teleport}.
To measure this overhead, we now compare the performance of optimization
methods with and without teleportation with respect to total compute time
as \( k \) is increased from \( 1 \) to \( 50 \).
\cref{fig:logistic-profile-time} shows a version of \cref{fig:logistic-profile}
(logistic regression) with time along the x-axis, while
\cref{fig:network-profile-time} gives a version of
\cref{fig:network-profile} (three-layer ReLU network).

We see similar trends for both the convex logistic regression problems and
non-convex neural network training problems.
When \( k = 1 \), teleportation has minimal overhead but performance is essentially
unchanged from the default gradient methods.
Increasing \( k \) to \( 10 \) or \( 25 \) increases the time cost, but with
only marginal gains.
However, for \( k = 50 \), teleportation methods exceed or match the performance
of methods without teleportation while using the same time budget.
This trend is particularly noticeable for three-layer ReLU networks
in the stochastic setting.
We conclude that (i) solving the teleportation sub-problem to sufficient accuracy
is critical for good performance and (ii) methods with teleportation are more
effective than standard methods given sufficient time budget.


\textbf{Stochastic and Non-Smooth Image Classification}:
We replicate our experiments on MNIST with stochastic gradients
using a batch-size of 128 in \cref{fig:mnist-stochastic}.
The experimental setup is otherwise identical to that for
\cref{fig:mnist-convergence}.
In this stochastic setting, we find almost no difference between methods with
or without teleportation.
Indeed, all methods appear to converge (noisily) to similar performing models.

\begin{figure}[t]
    \centering
    \includegraphics[width=0.98\textwidth]{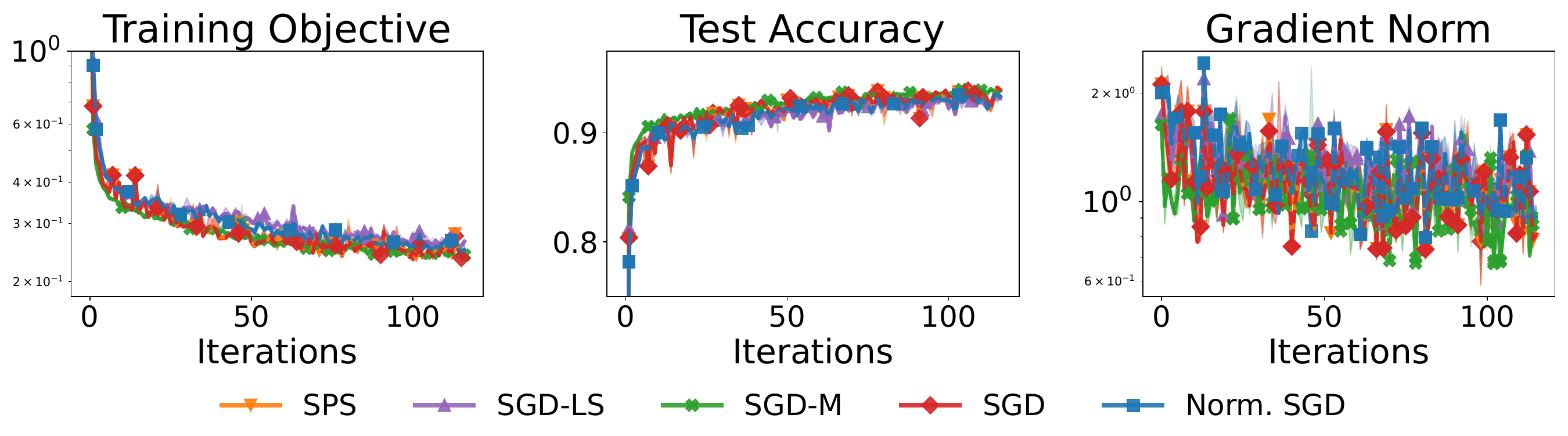}
    \caption{
        Performance of optimizers with (solid) and without (dashed)
        teleportation on MNIST.
        We train a MLP with the soft-plus activation and one hidden layer of size \(
        500 \).
        All methods are run with batch-size of \( 128 \).
        Teleportation does not appear to help or harm optimization.
    }%
    \label{fig:mnist-stochastic}
\end{figure}

We also consider training non-smooth ReLU neural networks on MNIST.
Although we did not find a meaningful difference between non-smooth and smooth
activations in our experiments on the UCI datasets%
\footnote{Methods with
    teleportation performed very slightly worse compared to those without
    teleportation when using the soft-plus activation on the UCI datasets, but the
    overall trend in performance profiles was identical.
},
smooth activation functions appear more important for teleportation on MNIST.
This is consistent with our theory, which requires \( f \) to be
differentiable and \( L \)-smooth.

\cref{fig:mnist-deterministic-non-smooth} shows the results of training
a two-layer ReLU neural network with 500 hidden units.
While SGD and the Polyak step-size (SPS) still perform better with
teleportation, SGD-LS stalls and SGD-M is quite unstable when used with
teleportation.

\begin{figure}[t]
    \centering
    \includegraphics[width=0.98\textwidth]{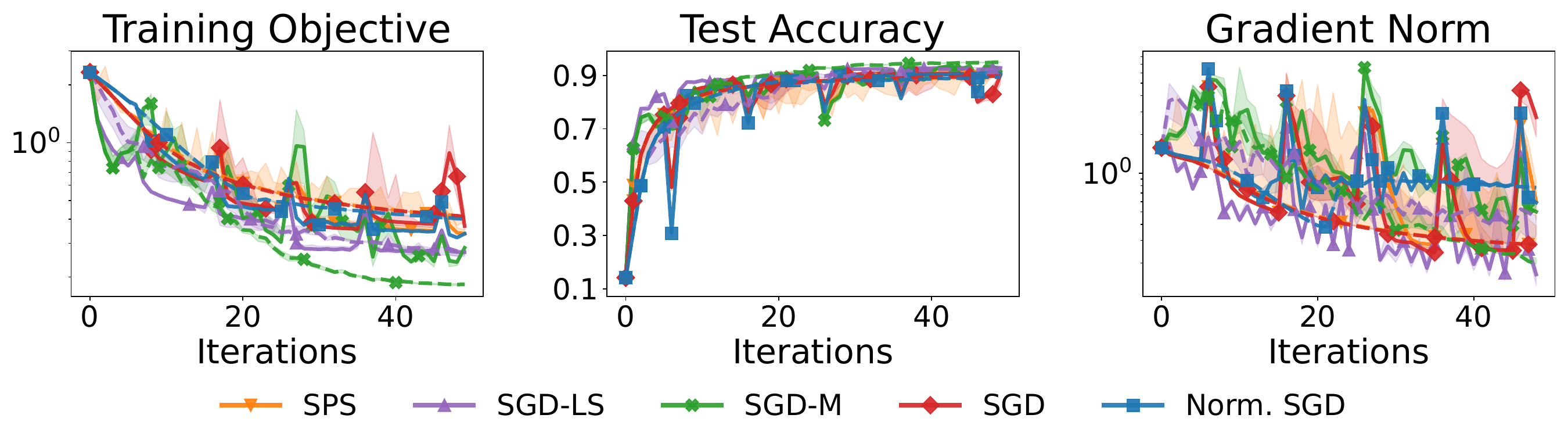}
    \caption{Convergence of deterministic gradient methods for training a
        two-layer ReLU network with 500 hidden units on the MNIST dataset.
        Solid lines show methods with teleportation, while dashed lines show the same
        algorithms without teleportation.
        Unlike when using smooth activations, teleportation has mixed effectiveness.
        While SGD and the Polyak step-size (SPS) still perform better with
        teleportation, SGD-LS stalls and SGD-M is quite unstable when used with
        teleportation.
    }%
    \label{fig:mnist-deterministic-non-smooth}
\end{figure}

\textbf{Full Comparison with Truncated Newton Method}:
Here we include additional details for the experimental comparison between
GD with teleportation at every step and Newton-CG.
In particular, we training curves for the test accuracy and gradient
norm which were omitted from the main paper.
These results further confirm our theoretical findings.

\cref{fig:newton-comparison-logreg} presents full results for three convex
logistic regression problems. While Newton-CG enjoys fast
convergence initially, it stalls on the Hill Valley and Chess datasets.
This is consistent with existing results on truncated Newton methods.
For example, the linear rate proved by \citet{karimireddy2019newton}
only holds under a hard-to-verify assumption on the quality of the Hessian
approximation.
Since Hill Valley has $d = 100 \gg 25$ features, this
assumption likely fails and Newton-CG performs poorly.
In contrast, GD with teleportation shows robust, linear convergence.
This linear convergence rate confirms our theoretical results in
Theorems~\ref{thm:stability-teleport-ls} and~\ref{thm:stability-combined-ls}.

\cref{fig:newton-comparison-network} presents full results for the non-convex
neural network training problem.
We find that Newton-CG fails on all three problems, likely because it is
attracted to any critical point, including saddles and local maxima.
This replicates well-known failure modes of Newton-type methods for non-convex
optimization \citep{xu2020newton}.
In comparison, GD with teleportation still obtains a linear rate.
The robustness of teleportation for non-convex optimization is one of our
primary motivations for this paper.

\begin{figure}
    \centering
    \includegraphics[width=0.8\textwidth]{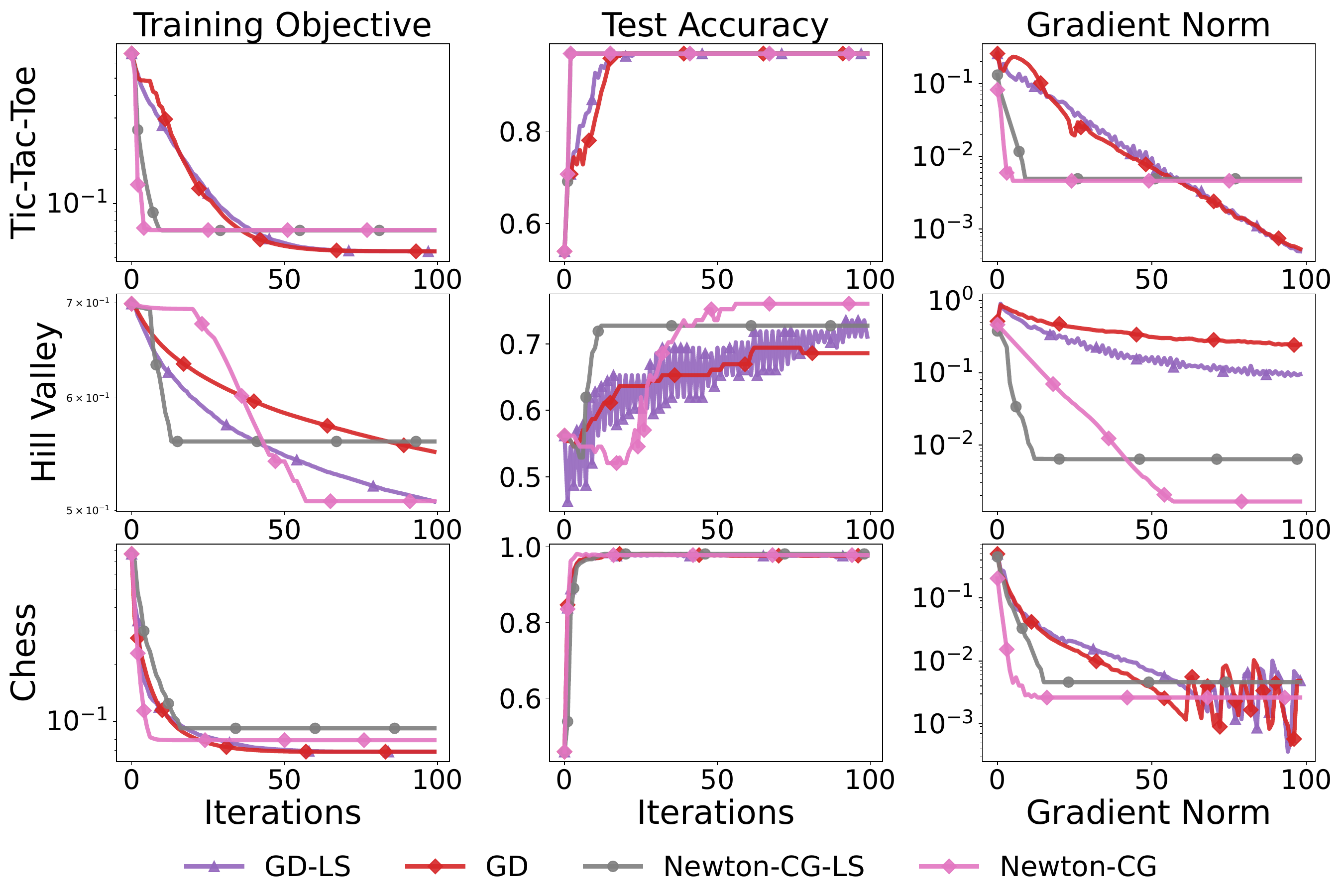}
    \caption{Comparison of inexact Newton and GD with teleportation using tuned
        step-sizes and using line-search (LS) for three logistic regression
        problems satisfying Hessian stability. The inexact Newton step is computed
        using conjugate gradients. For fairness, Newton-CG and GD with teleportation
        are limited to 25 Hessian-vector products at each step. While inexact Newton
        initially converges quickly, it is sensitive to the eigenvalue spectrum
        of the Hessian and convergence stalls on two out of three datasets.}%
    \label{fig:newton-comparison-logreg}%
\end{figure}

\begin{figure}
    \centering
    \includegraphics[width=0.8\textwidth]{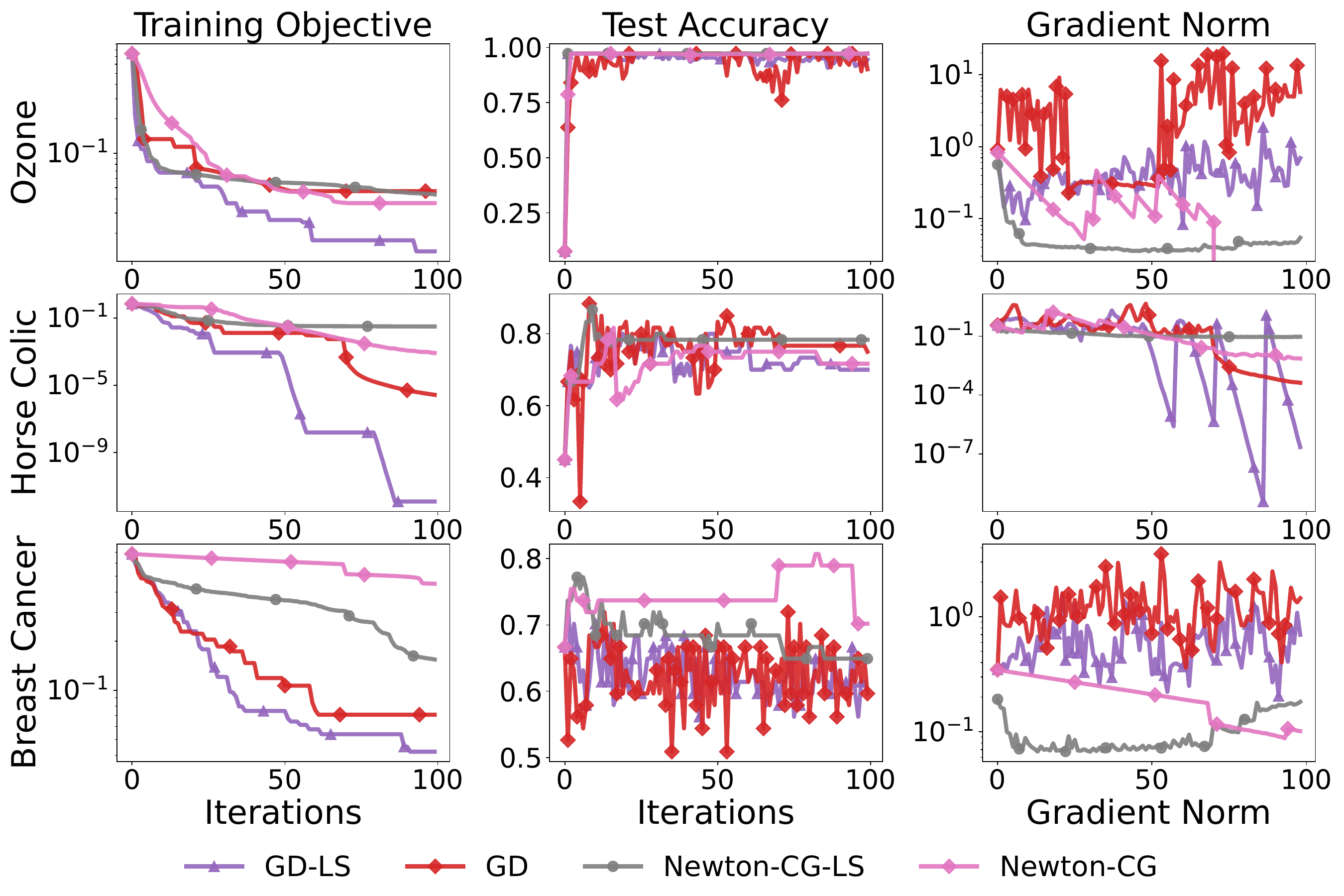}
    \caption{Comparison of inexact Newton and GD with teleportation for
        training two-layer ReLU networks. Both methods are limited to 25
        Hessian-vector products. We find that GD with teleportation converges at a
        fast rate while Newton-CG fails on every problem due to non-convexity.}%
    \label{fig:newton-comparison-network}
\end{figure}

\textbf{Convergence Curves for UCI Datasets}:
Finally, we provide convergence plots for a subset of problems from our
performance profiles on the UCI datasets.
We select these problems to illustrate various points about the effectiveness
of teleportation as an optimization sub-routine.
\cref{fig:uci-convergence-logreg} shows convergence for three logistic
regression problems on the \texttt{Chess}, \texttt{Ionosphere}, and
\texttt{Tic\-Tac\-Toe} datasets.
Note that we use weight decay regularization with strength \( \lambda =
10^{-6} \) and consider deterministic optimization (the plots for stochastic
optimization were very noisy).

We also provide convergence plots for training ReLU networks on the
\texttt{Credit Approval}, \texttt{Pima}, and \texttt{Ringnorm} datasets
(\cref{fig:uci-convergence-network}).
We use weight decay regularization with strength \( \lambda = 0.01 \) and
again consider only deterministic optimization.
All three training problems show a very interesting phenomena where several
methods with teleportation --- notably SGD and Normalized SGD --- actually
take bad steps after teleporting for which the objective increases.
Despite this, they converge to a similar or better final objective value than
methods without teleportation.
On \texttt{Pima}, they appear to convergence to a different local optimum,
implying that teleportation can affect the implicit regularization of
optimization methods.

\begin{figure}[t]
    \centering
    \includegraphics[width=0.98\textwidth]{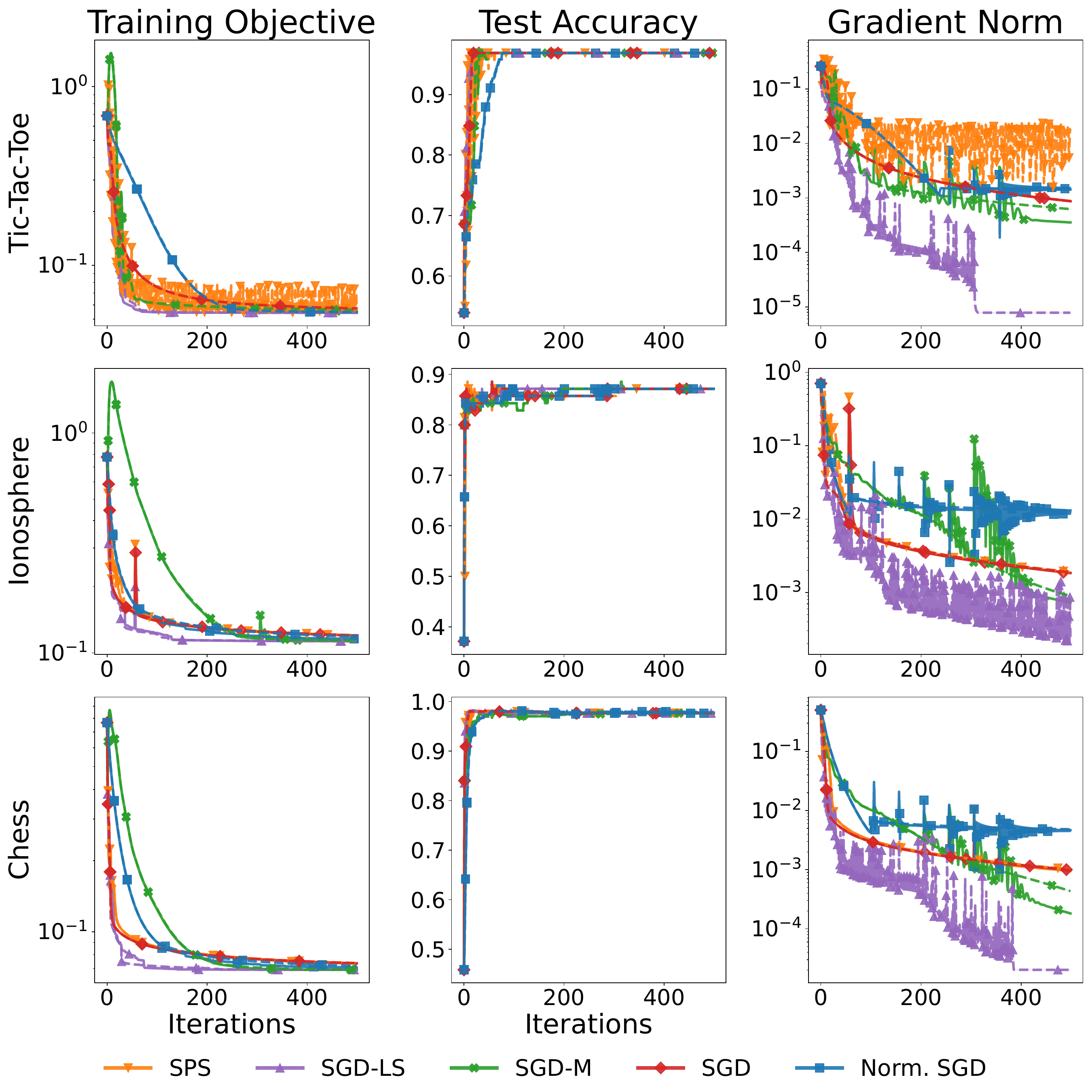}
    \caption{Convergence of deterministic gradient methods for training
        logistic regression models on three datasets from the UCI repository.
        Solid lines show methods with teleportation, while dashed lines show the same
        algorithms without teleportation.
        This figure reproduces exact training curves from \cref{fig:logistic-profile}.
    }%
    \label{fig:uci-convergence-logreg}
\end{figure}

\begin{figure}[t]
    \centering
    \includegraphics[width=0.98\textwidth]{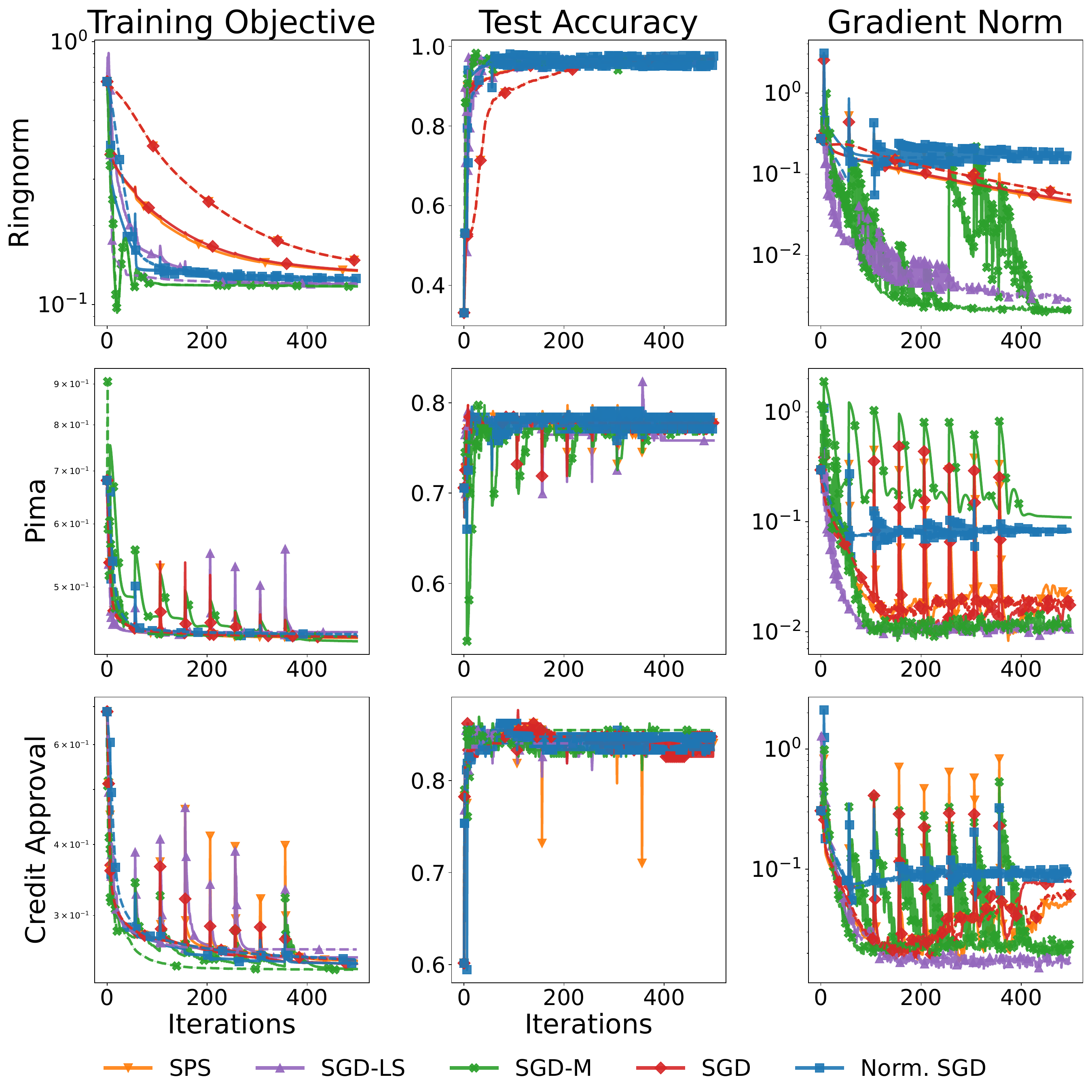}
    \caption{Convergence of deterministic gradient methods for three-layer
        ReLU networks with 100 units each on three datasets from the UCI repository.
        Solid lines show methods with teleportation, while dashed lines show the same
        algorithms without teleportation.
        This figure reproduces exact training curves from \cref{fig:network-profile}.
    }%
    \label{fig:uci-convergence-network}
\end{figure}

\subsection{Experimental Details}\label{app:experiment-details}

In this section we include additional details necessary to replicate our
experiments.
The code used to run the experiments in this paper is publicly available
on GitHub at
\url{https://github.com/aaronpmishkin/teleport}.
We run our experiments using PyTorch \citep{paszke2019pytorch}.
When selecting step-sizes for optimization methods, we perform a grid-search
using the grid \( \cbr{1000, 100, 10, 5, 2, 1, 10^{-1}, 10^{-2}, 10^{-3}, 10^{-4}} \).
All experimental results are shown for three random restarts excepting the
deterministic Newton experiments and the performance profiles, where averaging is not straightforward.
We plot the median and first/third quartiles.
Step-sizes are selected by minimizing the training loss at the end of the last
epoch.

For our teleportation method, we initialize the step-size at \( \rho = 0.1 \)
and use the tolerances \( \epsilon = \delta = 10^{-6} \).
In practice, we scale \( \gamma_t \) by \( 0.1 \) to bias the teleportation
solver towards more aggressive maximization of the gradient norm.
Unless otherwise stated, we limit the teleportation method to a maximum of \(
50 \) iterations, although it can terminate earlier based on the tolerances.
We allow the teleportation method to backtrack at most 25 iterations before returning
a default step-size of \( \rho_t = 10^{-16} \).

For SGD with momentum, we use the momentum parameter \( \beta = 0.9 \) and
dampening parameter \( \mu = 0.9 \).
For SGD-LS, we optimize over relaxation parameters for the Armijo line-search and
choose the best from the set \( \cbr{0.001, 0.01, 0.1, 0.5} \).
To ensure the step-size returned by line-search is always close to the maximum
attainable, at each iteration we increase the step-size until the Armijo
condition fails before backtracking.
We use forward-tracking and back-tracking parameters of \( 1.25 \) and \( 0.8
\), respectively.
We estimate \( f^* \) with zero for SPS.
All neural networks are initialized using the standard Kaiming initialization
\citep{he2015initialization}.

\textbf{Test Functions}:
We use open-source implementations of the Booth and Goldstein-Price functions
\footnote{\url{https://github.com/AxelThevenot/Python_Benchmark_Test_Optimization_Function_Single_Objective}}.
Gradient descent with and without teleportation are run with an Armijo
line-search starting from step-size \( \eta = 1 \).
Newton's method is run with a fixed step-size \( \eta = 0.8 \).

\textbf{Solving the Teleportation Problem}:
We use a two-layer ReLU network with fifty hidden units and soft-plus
activation function.
The strength of weight decay regularization is set at \( \lambda = 1.8 \).
We initialize the teleportation method with step-size \( \rho = 1 \), resulting
in slightly more aggressive steps.

\textbf{UCI Performance Profiles}:
We run on the following 20 binary classification datasets selected from the UCI
repository:
\texttt{blood},
\texttt{breast-cancer},
\texttt{chess-krvkp},
\texttt{congressional-voting},
\texttt{conn-bench-sonar},
\texttt{credit-approval},
\texttt{cylinder-bands},
\texttt{hill-valley},
\texttt{horse-colic},
\texttt{ilpd-indian-liver},
\texttt{ionosphere},
\texttt{magic},
\texttt{mammographic},
\texttt{musk-1},
\texttt{ozone},
\texttt{pima},
\texttt{tic-tac-toe},
\texttt{titanic},
\texttt{ringnorm},
\texttt{spambase}.
We use the pre-processed datasets provided by \citet{fernandez2014we},
although we do not use their splits since these are known to have test set
contamination.
To obtain 120 distinct problems, we also consider regularization parameters
from the grid \( \cbr{10^{-1}, 10^{-2}, 10^{-3}, 10^{-4}, 10^{-5}, 10^{-6}}
\).
For each problem, we estimate \( f(\wopt) \) by running SGD-M with teleportation
to obtain a very accurate solution.
For fairness, we estimate \( f(\wopt) \) separately for deterministic and
stochastic problems; in the deterministic setting, we run SGD-M for \( 2500 \)
iterations with teleportation every fifty iterations starting at \( k = 5 \).
In the stochastic setting, we run SGD-M for \( 250 \) epochs with teleportation
steps every \( 10 \) epochs starting at epoch \( 5 \).
In both cases, we select step-sizes independently for each problem
using the grid-search described above.

For the stochastic setting, we use batch-sizes of \( 64 \) and train for \(
100 \) epochs.
The teleportation schedule is \( \calT = \cbr{5, 15, 25, 35, 45, 55, 65, 75,
    85, 95} \) and we record metrics every five iterations.
In the deterministic setting, we run for \( 500 \) iterations and use
teleportation schedule \( \calT = \cbr{5, 55, 105, 155, 205, 255, 305, 355}
\).
We record metrics at every iteration.

We use different tolerances to generate our performance profiles.
All profiles in the stochastic setting use \( \tau = 0.05 \), while profiles
in the deterministic setting use \( \tau = 0.15 \) for logistic regression
and \( \tau = 0.1 \) for neural networks.
These tolerances were chosen so that the worst-performing method solved
approximately \( 50\% \) of problems.
As discussed, we estimate \( f(\wopt) \) separately in the stochastic
and deterministic settings using SGD with momentum and teleportation.
Thus, the estimate of \( f(\wopt) \) is not the same between
\cref{fig:network-profile} and \cref{fig:network-profile-init}.

\textbf{Initialization by Teleportation}:
All experimental details are as described for the other performance
profiles except that the teleportation schedule is \( \calT = \cbr{0} \)
for methods using teleportation.

\textbf{Timed Performance Profiles}:
All experimental details are as described for the other performance
profiles except that we run all methods for longer.
We run deterministic methods for \( 1000 \) iterations with teleportation
schedule \( \calT = \cbr{ 5 + 50k : k \in \bbN, k < 20 } \) iterations.
For stochastic methods, we run for \( 200 \) epochs with teleportation
schedule \( \calT = \cbr{ 5 + 10k : k \in \bbN, k < 20 } \).

\textbf{Image Classification}: We use a fixed strength of
\( \lambda = 10^{-2} \) for the weight decay regularization.
All other settings are as described above.
Depending on the experiment, we use either the soft-plus or ReLU activations.
All models are two-layer networks with \( 500 \) hidden units.
We plot the median and interquartile range out of three random restarts.

\textbf{Effects of Regularization}: The data for \cref{fig:regularization-ablation}
comes directly from the performance profiles in
\cref{fig:network-profile}.
All results are shown for ReLU networks with two hidden layers of \( 100 \)
neurons each.

\textbf{Additional UCI Plots}: The data for \cref{fig:uci-convergence-logreg}
comes directly from the performance profile in \cref{fig:logistic-profile}
and the data for \cref{fig:uci-convergence-network} comes from the performance
profile in \cref{fig:network-profile}.
All neural network results are shown for regularization parameter \( \lambda =
0.01 \), while logistic regression results use \( \lambda = 10^{-6} \).

\textbf{Comparison with Newton}:
We use the open source implementation of Newton-CG in PyTorch
by \citet{feinman2021newton}.
The number CG steps for each step of Newton-CG is fixed at \( 25 \).
Similarly, for GD with teleportation we limit the number of iterations used by
\cref{alg:teleport} to be \( 25 \).
This ensures both methods use the same number of Hessian-vector products.
For Newton-CG-LS, we set the step-size using line-search on the strong Wolfe
conditions.
For methods with fixed step-sizes, we pick the step-size by grid-search on
the training loss over the set
\[
    \{ \cbr{1000, 100, 10, 5, 2, 1, 10^{-1}, 10^{-2}, 10^{-3}, 10^{-4}} \}.
\]
We set the regularization parameter to \( 0 \) (unregularized) for both the
convex and non-convex problems.
The non-convex model is a three-layer ReLU network with \( 100 \) hidden units in
each hidden layer.

\subsection{Computational Details}\label{app:computational-details}

The experiments in \cref{fig:toy-teleport} were run on a MacBook Pro
(16 inch, 2019) with a 2.6 GHz 6-Core Intel i7 CPU and 16GB of memory.
All other experiments were run on a Slurm cluster with several different node
configurations.
For experiments requiring accurate timing, we used nodes with Nvidia A100 GPUs
with either 80GB or 40GB memory (their throughput is rated as the same)
and Icelake CPUs.
For experiments without accurate timing, we also allowed nodes with
Nvidia H100-80GB GPUs, Nvidia V100-32GB, or V100-16GB GPUs, where
nodes with the last two cards use Skylake CPUs instead of Icelake.
All jobs were allocated a single GPU and 24GB of RAM.
We thank the Scientific Computing Core, Flatiron Institute for support
running our experiments.

\end{document}